%% file: arxiv_main.tex
\documentclass[12pt]{article}

\usepackage[english]{babel}
\usepackage{csquotes}
\usepackage[a4paper,margin=1in]{geometry}

\usepackage[style=numeric-comp, sorting=none, backend=biber]{biblatex}
\usepackage{graphicx}

\usepackage{amsmath}
\usepackage{amsfonts}
\usepackage{amssymb}
\usepackage{amsthm}

\usepackage[caption=false,font=footnotesize]{subfig}

\usepackage{algorithm}
\usepackage{algpseudocode}
\usepackage[export]{adjustbox}

\usepackage{multirow}
\usepackage{placeins}
\usepackage{tikz}
\usetikzlibrary{arrows.meta,positioning,fit}

\usepackage[colorlinks=true, allcolors=blue]{hyperref}
\usepackage[capitalise,noabbrev]{cleveref}
\hypersetup{hypertexnames=false}

\newtheorem*{definition}{Definition}

\newtheorem*{remark}{Remark}
\theoremstyle{plain}
\newtheorem{corollary}{Corollary}
\theoremstyle{definition}
\title{Soft Tuy-Completeness for Robust Projection Selection in Cone-Beam CT}
\author{%
  Linda-Sophie Schneider and Andreas Maier \\
  \normalsize Pattern Recognition Lab, Friedrich-Alexander-Universit\"{a}t Erlangen-N\"{u}rnberg, Germany \\
  \normalsize \texttt{linda-sophie.schneider@fau.de}
}
\date{May 19, 2026}

\begin{document}

\maketitle
\thispagestyle{plain}

\begin{abstract}
This work introduces a continuous soft near-orthogonality score and a
resolution-aware saturated coverage objective for projection selection in
region-of-interest focused cone-beam CT, grounded in Tuy's completeness
theory. Replacing the binary hit-or-miss model of classical Tuy completeness
with a graded, differentiable formulation preserves a direct link to achievable
feature sizes while enabling both efficient approximate and exact optimisation.

We establish that the underlying discrete decision problems are NP-complete via
polynomial-time reductions from Set Cover, motivating a submodular greedy
algorithm with proven $(1-1/\mathrm{e})$ approximation guarantees and a
mixed-integer linear program (MILP) that provides certified optimality bounds.
The MILP serves as a quality certificate for the greedy solution rather than a
competing optimiser.

The primary empirical finding confirms this relationship: across a systematic
benchmark spanning six target regions, multiple projection budgets, and four
controlled occlusion conditions, the pooled median greedy-to-MILP objective
ratio was 0.998, with a substantial fraction of cases certified globally
optimal. A binary formulation is included as a diagnostic baseline; it
strengthens hard directional completeness but is weaker on the continuous
coverage scale.

We additionally introduce Effective Spatial Resolution (ESR), a physically
interpretable trajectory-level diagnostic that maps directional sampling gaps
to achievable feature sizes. ESR correlates reliably with matched
reconstruction quality across projection budgets and occlusion levels,
providing a practical bridge between the selection stage and the image domain
without requiring reconstruction.
\end{abstract}


\input{content/intro}

\input{content/trajectory_np_complete}

\input{content/methods}

\input{content/projection_selection_methods}

\input{content/experiments}

\input{content/results}

\input{content/ablation_study}

\input{content/discussion}

\input{content/conclusion}

\section*{Conflict of Interest Statement}
The authors declare that the research was conducted in the absence of any commercial or financial relationships that could be construed as a potential conflict of interest.

\section*{Ethics Statement}
This study did not involve human participants, animals, or identifiable personal data. No ethical approval was required for the computational research presented in this work.

\newpage

\printbibliography

\end{document}

%% file: content/intro.tex
\section{Introduction}


Cone-beam computed tomography (CBCT) has become an essential imaging modality across medical, industrial, and research applications. In medical imaging, CBCT is routinely deployed in interventional radiology suites, surgical navigation systems, and dentistry for real-time 3D guidance at a lower radiation dose and with faster acquisition than conventional multi-detector CT~\cite{Stayman2019Task}. In industrial contexts, cone-beam systems enable non-destructive evaluation and inspection of complex parts with spatial resolutions and material sensitivity often exceeding conventional CT~\cite{Herl2020Scanning,Bauer2024Scan}. However, despite these advantages, standard circular or helical scan trajectories remain suboptimal for many applications, especially when scanning large, complex, or non-convex objects, or when operating under mechanical and dose constraints \cite{Herl2021Task,Holub2019RoboCTA,Schneider2024IntegerOO}. Furthermore, standard circular trajectories fundamentally fail to satisfy Tuy’s data completeness condition for off-center regions of interest (ROIs), as the source path cannot intersect all Radon planes passing through the target volume, inevitably leading to cone-beam artifacts and reduced reconstruction fidelity~\cite{Hatamikia2022Source}.

A fundamental challenge in CBCT system design is the trajectory completeness problem: ensuring that the acquired projections contain sufficient information for accurate 3D reconstruction. The classical answer is provided by Tuy's data completeness condition~\cite{Tuy1983Inversion,Smith1985Image}, which establishes that exact reconstruction is possible if and only if every plane intersecting the object also intersects the X-ray source path. In Radon-space terms, this corresponds to complete sampling of the 3D Radon transform over a region containing the object. While theoretically elegant, Tuy's condition is formulated over a continuous parameter space where every plane through the object must intersect the source path, but its per-plane hit-or-miss character makes it effectively binary: a trajectory is either complete or it is not, with no gradient or degree of coverage to guide optimization. Maier et al.~\cite{Maier2015Discrete} and Liu et al.~\cite{Liu2012Completeness} later introduced practical discrete approximations by sampling the unit sphere of plane normals and counting covered directions, converting the continuous criterion into a quantitative but still binary-per-direction coverage score. This voxel-centric discretisation is computationally tractable, yet it still assumes idealized detector geometry and perfect visibility, and provides limited guidance for designing practical trajectories subject to real-world constraints.

In practice, modern CBCT systems must contend with numerous physical and engineering realities: finite detector size and resolution, cone-beam geometry, spatially varying X-ray attenuation through the object and surrounding materials, scattered radiation, mechanical limits of robotic arms or C-arm gantries, and clinical or industrial dose and acquisition-time budgets. These constraints make it infeasible to fully achieve Tuy completeness; instead, practitioners must make informed trade-offs between trajectory complexity, data acquisition cost, and reconstruction quality. Furthermore, in many applications, perfect reconstruction of the entire volume is unnecessary; instead, the goal is to achieve high-quality imaging of a specific ROI such as a surgical target, a critical component, or a suspected defect. Such ROI-focused imaging offers the potential for substantial reductions in scan time and radiation dose.

\input{content/related_work.tex}

\subsection{Contributions and Organization}

This work advances the theory and practice of cone-beam CT trajectory design by developing a physics-based discrete optimization framework for projection selection, grounded in completeness theory and equipped with formal complexity and algorithmic guarantees. Our key contributions are as follows:

\begin{enumerate}
  \item \textbf{Soft coverage formulation and resolution-aware completeness model.} The primary modeling contribution is a continuous \emph{soft near-orthogonality score}: a smooth, differentiable relaxation of the binary per-direction Tuy hit criterion that replaces the hard $0/1$ coverage decision with a graded measure of directional support. Building on this, we develop a resolution-driven, voxel-wise completeness formulation that integrates Nyquist-based angular sampling requirements, attenuation constraints, and detector visibility into a unified saturated coverage objective, which directly serves as the optimization target.

  \item \textbf{Physically interpretable validation metric.} We introduce the \emph{Effective Spatial Resolution (ESR)} as a task-agnostic, physics-based metric that maps angular sampling gaps directly to resolvable feature sizes within the ROI. ESR provides an interpretable validation tool that is decoupled from the optimization objective and enables principled assessment of trajectory quality independent of reconstruction algorithms.

  \item \textbf{Complexity-theoretic characterization of ROI-based trajectory optimization.} We formally define the ROI-based CT trajectory optimization problem (ROI-CTTOP) and prove that its decision variant is NP-complete via polynomial-time reduction from Set Cover. A corollary extends the decision-problem result to binary and soft directional coverage, while the corresponding optimization variants are NP-hard. This provides a rigorous complexity-theoretic foundation for completeness-driven CT trajectory design and motivates the use of both efficient greedy approximations and exact methods that provide optimality certificates.

  \item \textbf{Greedy algorithm with approximation guarantees and MILP-based certification.} Leveraging the submodular structure of the completeness objective, we present a marginal-gain greedy projection-selection algorithm with $\mathcal{O}(kmn)$ complexity and provable $(1-1/e)$ approximation guarantee as the primary practical solver. To assess how tight this guarantee is on realistic CT instances, we additionally develop an exact MILP formulation that provides certified optimality bounds via branch-and-cut. The certificates show that greedy is already near-optimal in practice, establishing the MILP primarily as a diagnostic benchmark rather than a competing optimizer.

\end{enumerate}

The paper is organized as follows. \cref{sec:roi_cttop} formalizes the ROI-CTTOP problem and establishes its NP-completeness, connecting it to classical maximum-coverage problems. \cref{sec:completeness} develops the resolution-aware completeness model and associated metrics. \cref{sec:algorithms} presents the greedy and MILP-based optimization algorithms. \cref{sec:experiments} describes the experimental setup, compared metrics, and reconstruction protocol. \cref{sec:results} reports the main selection and reconstruction results, \cref{sec:ablation} analyzes the targeted ablations, \cref{sec:discussion} discusses implications and limitations, and \cref{sec:conclusion} closes with the main takeaways and future work.

%% file: content/related_work.tex
\subsection{Related Work and Problem Positioning}

Over the past two decades, several research directions have emerged to address trajectory design in CBCT under practical constraints. We situate our work within this landscape by critically reviewing four complementary research areas and identifying key gaps that motivate our contributions.

\paragraph{Voxel-wise and local completeness metrics.}

To operationalize Tuy's condition in realistic systems, researchers have developed quantitative, voxel-level measures of data completeness. The foundational direction-based sampling approach proposed by Maier et al.~\cite{Maier2015Discrete} and Liu et al.~\cite{Liu2012Completeness} reformulates Tuy's global requirement locally: rather than requiring every Radon plane to intersect the source trajectory, they discretize the unit sphere of plane normals and check whether each sampled direction is covered by at least one source position. This voxel-centric perspective is both computationally practical and amenable to optimization.

Building upon this foundation, Herl et al.~\cite{Herl2020Scanning,Herl2022X} introduced significant extensions: they incorporate attenuation-based filtering by neglegting rays with high mean attenuation and angular-distance weighting, yielding more robust quality measures that account for ray penetration and material absorption. Their attenuation-aware completeness metrics are particularly relevant to materials with high density or high-Z components, where naive direction sampling would overestimate the utility of poorly-penetrating rays. Contemporaneous work by Clackdoyle and Noo~\cite{Clackdoyle2020Quantification} characterized incompleteness at the voxel level and validated local Tuy conditions using phantom studies, further demonstrating the practical utility of voxel-based completeness measures.

However, a critical limitation of existing completeness-based work is its reliance on greedy heuristics. Standard approaches iteratively select the source that provides the largest marginal gain in coverage, but this strategy is sensitive to initialization order, unable to refine or revisit previously selected views, and lacking optimality guarantees. While greedy selection provides a scalable baseline, it foregoes the opportunity to benchmark solution quality against globally optimal trajectories. Our work directly addresses this gap by embedding these robust completeness metrics within a rigorous combinatorial optimization framework that leverages both greedy approximation algorithms with proven $(1-1/\mathrm{e})$ guarantees and mixed-integer linear programming (MILP) to provide certified optimality bounds and confirm that the greedy approximation is already near-optimal on the studied instances.

\paragraph{Task-driven and object-aware trajectory optimization.}

A complementary paradigm to completeness-driven design is task-based image quality optimization, wherein the goal is to maximize performance on a specific diagnostic or measurement task rather than achieve universal data completeness. In medical imaging, Stayman et al.~\cite{Stayman2019Task,Capostagno2019Task} developed task-driven trajectory optimization frameworks that directly optimize model-observer performance, namely the detectability index for a lesion detection task. Related work by Herl et al.~\cite{Herl2021Task} extended these approaches to industrial settings with multiple features of interest.

In industrial CT, object-aware methods have proven particularly effective. Fischer et al.~\cite{Fischer2016Object} proposed a detectability-index-based optimization algorithm for industrial X-ray tomography that leverages CAD models to determine optimal acquisition poses for user-defined features, reducing the number of required projections by up to 75\% compared to standard circular scans. Parallel efforts by Hatamikia et al.~\cite{Hatamikia2020Optimization,Hatamikia2021Toward,Hatamikia2022Source} developed target-based CBCT frameworks incorporating kinematic constraints from C-arm gantries and surgical equipment, demonstrating clinically feasible customized trajectories with significantly reduced projection counts.

Recent deep learning approaches have further enhanced task-driven optimization. Yuan et al.~\cite{Yuan2024ApplicationOG} propose recurrent neural networks (RNNs) for sequential view selection, while Yang et al.~\cite{Yang2022Learning} introduce learning-based approaches to estimate projection importance through reconstruction networks. While these task-driven and learning-based methods are powerful and increasingly practical, they operate in continuous orbit-parameter spaces like C-arm pose angles or source-detector distances where local minima and non-convexity are endemic. Moreover, task-driven methods, by definition, optimize for a specific diagnostic goal and may be suboptimal for imaging scenarios where the clinical task is unknown or variable.

Our approach is orthogonal and complementary: we reformulate the problem as a \emph{discrete view-selection task}, in which the candidate source positions are treated as a finite set and the goal is to choose a subset. This discrete formulation enables the application of exact integer programming solvers that provide optimality guarantees and upper-bound certificates. By focusing on geometric completeness rather than a specific task, our framework establishes a physics-based ``floor'' of data sufficiency that holds across diverse diagnostic and measurement objectives.

Like Fischer et al., our attenuation-based validity model requires forward projections of the object — in practice derived from a CAD model or a reference scan. However, our framework requires only a geometric region specification (the ROI) rather than a task-specific detectability model tied to a particular feature. Moreover, the geometric completeness component of our framework can be applied independently of the attenuation model when object-specific priors are unavailable, making the approach applicable across a broader range of scenarios.

\paragraph{Sparse-view CT and reconstruction-aware view selection.}

Orthogonal to the view-selection problem, substantial research effort has been devoted to \emph{sparse-view reconstruction}, exploring how iterative algorithms~\cite{Natterer2001Mathematics} and learning-based methods can compensate for incomplete or sparse sampling by leveraging prior knowledge, regularization, or neural network priors. Di et al.~\cite{Di2023Review} provide a comprehensive review of sparse-view and limited-angle CT reconstruction using deep learning, documenting recent advances in unsupervised methods~\cite{Chang2025Unsupervised}, Gaussian representations~\cite{Wang2024Sparse}, and hybrid deep-iterative approaches~\cite{Huang2019Data}.

Recent advances in differentiable reconstruction, such as DRACO~\cite{Ye2025DRACO}, enable accurate reconstruction of arbitrary non-circular orbits and provide differentiable loss functions for learning-based trajectory optimization. More directly relevant to explicit view planning, recent work has proposed data-driven view-selection strategies. Lin et al.~\cite{Lin2025Tomographic} introduce the \emph{View Covariance Loss}, which optimizes a statistical surrogate for reconstruction error by studying the covariance structure of the measurement noise in Radon space. Gardner et al.~\cite{Gardner2019Improvements} and other researchers have demonstrated substantial improvements in CBCT image quality using novel iterative reconstruction algorithms, highlighting the potential of reconstruction-side advances.

However, a fundamental distinction separates reconstruction-aware and completeness-driven approaches: the former assume that advanced reconstruction algorithms can \emph{compensate} for geometric deficiencies (e.g., by exploiting sparsity priors or learned image models), while the latter ensures data \emph{sufficiency} from first principles. Our framework complements these reconstruction-side advances by ensuring that selected views satisfy the fundamental physics-based requirements of data completeness before reconstruction algorithms are applied. This two-stage perspective (first ensure geometric completeness, then apply task-optimized reconstruction) provides a principled baseline and is particularly valuable in clinical settings where reconstruction algorithm availability or performance may be variable.

\paragraph{Sampling theory and resolution requirements in CT.}

Beyond qualitative direction coverage, the classical theory of angular sampling in CT provides quantitative requirements for spatial resolution. Building on work by Tuy~\cite{Tuy1983Inversion}, Smith~\cite{Smith1985Image}, and Natterer~\cite{Natterer2001Mathematics}, the Nyquist-Shannon sampling theorem applied to the angular domain establishes that features of linear size $f_{\min}$ within a circular ROI of radius $r$ require angular sampling density proportional to $r/f_{\min}$. Specifically, the maximum allowable angular gap between adjacent source positions is bounded by $\theta_{\max} \approx f_{\min}/(2r)$~\cite{buzug2011einfuhrung}.

Prior work on sparse angular sampling by Dennerlein et al.~\cite{Dennerlein2005Exact}, Yaz\i{}c\i{} et al.~\cite{Yazici2012Inversion}, and Kazantsev et al.~\cite{Katsevich2013exact} has explored incomplete orbits and their impact on resolution, often using Fourier-domain analysis to characterize aliasing and blur. Classical reconstruction algorithms including Grangeat's method~\cite{Grangeat1991Mathematical,Defrise1994cone} and direct Fourier inversion~\cite{Schaller1998efficient} rely on completeness of 3D Radon data, which requires proper angular sampling. However, much of this work focuses on continuous orbits or parametric trajectories (e.g., circular, helical, or spiral) rather than discrete view selection.

Herl et al.~\cite{Herl2022X} have established a completeness condition for sets of arbitrary projections that integrates the Nyquist angular sampling rate $\Delta\gamma = f_{\min}/(2r)$ directly into a per-direction binary coverage check, which was the foundation of our work. Our contribution is twofold: we replace the binary hit-or-miss test with a continuous \emph{soft near-orthogonality score} that provides graded directional coverage and enables a saturated coverage objective, and we embed this formulation within a discrete optimization framework with formal complexity and approximation guarantees.

\paragraph{Algorithmic and optimization-theoretic foundations.}

The mathematical structure of view selection is intimately connected to the \emph{maximum coverage problem}, a classical problem in combinatorial optimization. Given a universe $U$ of elements, a collection of subsets $\{S_1, \ldots, S_m\}$, and a budget $k$, the maximum coverage problem seeks $k$ subsets whose union has maximum cardinality. Nemhauser, Wolsey, and Fisher~\cite{Nemhauser1978analysis} proved that the set function is monotone and submodular, and that greedy selection attains a $(1-1/\mathrm{e})$-approximation. Hochbaum~\cite{Hochbaum1982Approximation} further established NP-hardness of the optimization variant. More recent developments in submodular maximization by Vondrák~\cite{Vondrak2008Optimal}, Filmus and Ward~\cite{Filmus2014Monotone} have refined approximation bounds and provided practical algorithms for matroid-constrained and cardinality-constrained variants.

Despite this rich theoretical foundation, the formal complexity of completeness-driven CT trajectory optimization has remained underexplored in the imaging literature. While integer programming has seen recent application in related CT problems~\cite{Schneider2024IntegerOO}, explicit NP-completeness proofs and approximation guarantees for completeness-driven view selection are lacking. Our work bridges this gap by formally proving NP-completeness of the ROI-based trajectory problem (ROI-CTTOP) via reduction from Set Cover, thereby rigorously justifying the use of both efficient greedy approximations for real-time applications and exact branch-and-cut methods for offline planning.

Moreover, we leverage recent advances in mixed-integer linear programming (MILP) solvers: modern branch-and-cut implementations (e.g., Gurobi~\cite{gurobi}) can efficiently handle problems with hundreds of binary variables (source positions) and thousands of continuous variables (directional coverages) by exploiting problem structure, generating cutting planes, and employing effective heuristics. This enables practical optimization of realistic trajectory-selection instances on standard hardware.

\paragraph{Multi-ROI trajectory design.}

A significant gap in the literature is the treatment of \emph{multiple competing ROIs} within a single acquisition budget. In many clinical applications (e.g., surgical guidance for multiple anatomical targets) and industrial settings (e.g., inspection of multiple components or defects), a single scan must provide adequate imaging of several distinct regions. Prior work on completeness metrics and view selection typically addresses single ROIs; extending to multiple ROIs under a shared projection budget is non-trivial and has received limited systematic study.

Our MILP formulation directly addresses this gap by replicating coverage computations for each ROI and jointly optimizing a shared view set. As a structural by-product, the integer program can also incorporate feasibility constraints that exclude geometrically or mechanically infeasible source positions, though a systematic treatment of kinematic constraints is beyond the scope of this work and is left to future investigation.

\paragraph{Limited-angle tomography and artifact characterization.}

When view selection results in incomplete angular coverage, limited-angle artifacts arise that corrupt image quality. Frikel~\cite{Frikel2013Characterization} characterizes limited-angle artifacts using microlocal analysis and develops edge-preserving reconstruction methods adapted to incomplete geometry. Huang et al.~\cite{Huang2019Data} propose data-consistent artifact reduction (DCAR) by combining deep learning priors with iterative reconstruction, achieving significant improvements for limited-angle cone-beam geometry. Understanding these artifact mechanisms motivates the importance of ensuring geometric completeness through view selection, rather than relying solely on reconstruction-side mitigation.

%% file: content/trajectory_np_complete.tex

\section{ROI-Based CT Trajectory Optimization and Complexity Analysis}\label{sec:roi_cttop}

In many clinical and industrial applications of computed tomography, practitioners are primarily interested in obtaining high-fidelity reconstructions of a specific ROI rather than the entire volume. By focusing acquisition efforts on a critical subset of voxels, we can reduce scan time, radiation exposure, and data storage requirements, while still ensuring diagnostically useful images of the target structure~\cite{Stayman2019Task, Fischer2016Object, Bauer2024Scan}. In this section, we formalize the ROI-Based CT Trajectory Optimization Problem (ROI-CTTOP), discuss its combinatorial structure, and establish its computational intractability via an NP-completeness proof.

\subsection{Problem Formulation}

Let $V$ denote the finite set of all voxels in the 3D volume under consideration. We assume a discrete set of candidate X-ray projection angles, denoted by $P = \{p_1, p_2, \dots, p_m\}$. Each projection $p_i$ images a subset of voxels, which we write as $R_i \subseteq V$.  We also designate a region of interest $T \subseteq V$, typically corresponding to the anatomical or structural subvolume on which diagnostic or quantitative analysis is to be performed.

Our goal is to select at most $k$ projections from $P$ so as to maximize the number of ROI voxels that are covered at least once. The budget is therefore an upper bound on acquisition cost; if an implementation requires exactly $k$ acquired views, unused budget can be filled only after the optimization has reached zero marginal gain. Formally:

\begin{definition}[ROI-CTTOP]
Given a voxel set $V$, projections $P=\{p_1,\dots,p_m\}$ with coverage $R_i\subseteq V$, an ROI $T\subseteq V$, and a budget $k\in\mathbb{Z}_{>0}$, the ROI-CTTOP is
\[
  \max_{\substack{P'\subseteq P\\|P'|\le k}} \bigl|\bigl(\bigcup_{p_i\in P'}R_i\bigr)\,\cap\,T\bigr|.
\]
\end{definition}

In words, we choose up to $k$ angles whose union of rays intersects the largest possible number of voxels within the ROI. Note that voxels covered multiple times do not receive extra credit; it is a pure coverage metric. To analyze computational complexity, we next introduce the corresponding decision problem.

\begin{definition}[Decision ROI-CTTOP]
Given the same inputs as above, plus a coverage threshold $L\in\mathbb{Z}_{\ge0}$, the decision question is:
\[
  \text{Does there exist }P'\subseteq P,\;|P'|\le k,\quad\text{s.t.}\quad
  \bigl|\bigl(\cup_{p_i\in P'}R_i\bigr) \cap T\bigr| \;\ge\; L?
\]
\end{definition}

This binary formulation allows us to employ standard tools from computational complexity theory. In particular, we will show that deciding whether one can achieve a specified level of ROI coverage with a budget of $k$ views is an NP-complete problem.

\subsection{NP-Completeness of Decision ROI-CTTOP}\label{sec:complexity}

We now prove that the decision variant of ROI-CTTOP is NP-complete. The proof proceeds via the canonical two-step procedure: establishing membership in NP and demonstrating NP-hardness through a polynomial-time reduction from the Set Cover problem.

\paragraph{Membership in NP.}
To establish that Decision ROI-CTTOP lies in NP, we demonstrate that a valid solution can be verified in polynomial time. A certificate for this problem is a subset $P' \subseteq P$ consisting of at most $k$ projections. Given such a certificate, a deterministic verification algorithm first confirms the cardinality constraint $|P'| \le k$. It then initializes a boolean coverage map for the voxel set $V$ and iterates through each selected projection $p_i \in P'$, marking all voxels $v \in R_i$ as covered. Finally, the algorithm computes the intersection of the covered set with the region of interest $T$ and checks if the cardinality meets the threshold $L$. Since the summation of coverage sets and the counting of ROI voxels require at most $\mathcal{O}(m|V|)$ operations, the verification runs in time polynomial in the input size.

\paragraph{NP-Hardness via Set Cover.}
We establish hardness by reducing from the classic Set Cover problem, which is known to be NP-complete~\cite{Garey1979ComputersAI}. The problem is defined as follows:

\begin{quote}
  \textbf{Instance:} A finite universe $U$ of size $n$, a collection of subsets $\mathcal{S}=\{S_1,\dots,S_m\}$ where $S_i\subseteq U$, and a selection budget $k \in \mathbb{Z}_{>0}$.  \\[0.3em]
  \textbf{Decision:} Does there exist a subcollection $\mathcal{S}'\subseteq\mathcal{S}$ with $|\mathcal{S}'| \le k$ such that $\bigcup_{S \in \mathcal{S}'} S = U$?
\end{quote}

We construct a reduction that maps any instance of Set Cover to an instance of Decision ROI-CTTOP in polynomial time. Let $V := U$ and $T := U$, so that the ROI comprises the entire universe of elements. We create projections $P=\{p_1,\dots,p_m\}$ by setting $R_i := S_i$ for each $i$, so every projection corresponds exactly to one set in the cover instance. We retain the same budget $k$ and set the coverage threshold $L := |U| = n$. This transformation runs in time polynomial in $n$ and $m$.

With the construction complete, we now demonstrate that a solution exists for the Set Cover instance if and only if a valid solution exists for the constructed ROI-CTTOP instance. We verify this logical equivalence in two steps:

\begin{itemize}
  \item[\textbf{$\Rightarrow$}] First, assume the original Set Cover instance has a solution. If there exists a subcollection $\mathcal{S}'$ of size at most $k$ covering all of $U$, then selecting the corresponding projections $P'=\{p_i: S_i\in\mathcal{S}'\}$ yields a total coverage of
  \begin{align*}
    \bigl|\bigl(\cup_{p_i\in P'}R_i\bigr)\cap T\bigr| &= \bigl|\cup_{S_i\in\mathcal{S}'}S_i\bigr| \\
    &= |U| = L,
  \end{align*}
  where the second equality uses the assumption that $\mathcal{S}'$ covers all of $U$, and the third follows from $L := |U|$ by construction. Since the coverage threshold $L$ is met, the Decision ROI-CTTOP instance answers \emph{yes}.

  \item[\textbf{$\Leftarrow$}] Conversely, assume the constructed ROI-CTTOP instance has a solution. Let $P'$ be any subset of at most $k$ projections that covers at least $L=|U|$ ROI voxels. Since the total size of the ROI is $|U|$, this implies that the projections must cover every element in $U$. Consequently, the selected projections correspond to a subcollection $\mathcal{S}'\subseteq\mathcal{S}$ of size \emph{at most} $k$ that covers all of $U$, forming a valid set cover for the original instance.
\end{itemize}

Thus we have shown Set Cover $\le_p$ Decision ROI-CTTOP, establishing NP-hardness. Along with membership in NP, this proves Decision ROI-CTTOP is NP-complete. Finally, since solving the optimization variant of ROI-CTTOP would allow us to decide the decision problem, the optimization version is NP-hard.

The theorem above is stated for the voxel-coverage formulation of ROI-CTTOP. The optimization method developed in \cref{sec:completeness,sec:algorithms} operates instead on a directional coverage objective, formulating completeness in terms of sampled plane-normal directions rather than individual voxels. The following corollary shows that the same hardness conclusion holds for both the binary and soft directional coverage decision problems.

\begin{corollary}[NP-completeness of directional coverage]\label{cor:directional_npc}
Both the binary directional coverage decision problem and the soft directional coverage decision problem are NP-complete.
\end{corollary}

\begin{proof}
\textit{Binary directional coverage.} Given a binary coverage matrix $B \in \{0,1\}^{m \times z}$, a budget $k$, and a threshold $L$, the decision problem asks whether there exists a subset of at most $k$ projections covering at least $L$ of the $z$ sampled plane-normal directions. Membership in NP is immediate: a certificate $P' \subseteq P$ with $|P'| \le k$ is verifiable in $\mathcal{O}(mz)$ time by marking covered directions and counting them.

For NP-hardness we reduce from Set Cover. Given an instance $(U,\mathcal{S},k)$ with $\mathcal{S} = \{S_1,\dots,S_m\}$, set $z := |U|$, define $B_{ij} := \mathbb{I}[j \in S_i]$, keep the budget $k$, and set $L := z$. The construction runs in $\mathcal{O}(mz)$ time. A subcollection $\mathcal{S}' \subseteq \mathcal{S}$ of size at most $k$ covers all of $U$ if and only if the corresponding at most $k$ projections cover all $z$ directions, so the directional coverage instance answers \emph{yes} if and only if the Set Cover instance does. Hence binary directional coverage is NP-complete.

\textit{Soft directional coverage.} Given a matrix $A \in [0,1]^{m \times z}$ with rational entries of polynomially bounded encoding length, a budget $k$, and a threshold $\theta$, the decision problem asks whether there exists $x \in \{0,1\}^m$ with $\sum_i x_i \le k$ such that
\[
  \sum_{j=1}^{z} \min\!\left(1,\, \sum_{i=1}^{m} A_{ij} x_i\right) \;\ge\; \theta.
\]
Membership in NP holds because the objective can be evaluated and compared to $\theta$ in time polynomial in $m$ and $z$.

Binary directional coverage is the restriction to $A_{ij} = B_{ij} \in \{0,1\}$ with $\theta := L$. Under binary inputs, $\min(1, \sum_i B_{ij} x_i) = 1$ if and only if at least one selected projection covers direction $j$, so the objective equals exactly the number of covered directions. A polynomial-time algorithm for soft directional coverage would therefore decide binary directional coverage, which is NP-hard by the reduction above. Hence soft directional coverage is NP-complete.
\end{proof}

\subsection{Relation to Maximum Coverage and Submodularity}

The ROI-CTTOP objective can be viewed as a special case of the classic maximum coverage problem, where each projection $p_i$ corresponds to a set $R_i\cap T$ and the goal is to select up to $k$ sets whose union has maximum cardinality~\cite{Nemhauser1978analysis}. The corresponding set function
\[
  f(P') \;=\; \bigl|\bigl(\cup_{p_i\in P'}R_i\bigr)\cap T\bigr|
\]

is monotone and submodular, which implies that a simple greedy selection algorithm attains a $(1-1/\mathrm{e})$-approximation in the worst case~\cite{Nemhauser1978analysis,Hochbaum1982Approximation}.  This places ROI-CTTOP squarely within the broader family of submodular maximization problems under a cardinality constraint, and motivates the use of greedy projection selection as a scalable baseline.

At the same time, submodularity also highlights the potential benefits of combinatorial optimization beyond purely heuristic approaches. For moderate problem sizes, MILP solvers can exploit the combinatorial structure of $f$ to deliver certified optimality bounds and, when a proof of global optimality is found, exact solutions. In this work, we leverage this connection by combining a completeness-driven set function with both greedy selection and a MILP-based formulation, enabling us to compare fast approximate solutions against globally optimal trajectories and to quantify the associated coverage gaps in realistic CT scenarios.

%% file: content/methods.tex

\section{Quantitative Data Completeness Modeling}\label{sec:completeness}
The design of trajectories in cone-beam CT exists within the balance of completeness theory and practical constraints, such as finite detector support, limited resolution, strong attenuation, and hard mechanical limits. Although Tuy's theorem provides a continuous geometric condition on the source path, it only acts as a feasibility criterion and does not quantify how much information a constrained trajectory provides for a given ROI. Therefore, we reformulate completeness locally as voxel-wise directional coverage following \cite{Maier2015Discrete}, making the sampling density explicit via Nyquist-based resolution requirements. We introduce a binary coverage matrix as the formal discrete counterpart of Tuy's condition and a continuous soft relaxation thereof as the primary optimization objective. A separate resolution-based diagnostic metric is defined to validate trajectory quality independently of the optimization model.

\subsection{From Tuy Completeness to Discrete Directional Coverage}\label{sec:tuy_to_discrete}

Tuy's theorem~\cite{Tuy1983Inversion} establishes necessary and sufficient conditions for exact cone-beam reconstruction: a compactly supported object $f(\mathbf{x})$ can be exactly reconstructed if and only if every plane intersecting its support also intersects the source trajectory $\mathcal{S}$. In Radon-space terms, this requires complete sampling of the 3D Radon transform over a region containing the support. Stated voxel-wise, for every point $v \in \Omega$ and every plane-normal direction $\mathbf{u} \in \mathbb{S}^2$, there must exist a source $s \in \mathcal{S}$ whose ray to $v$ is orthogonal to $\mathbf{u}$.
\begin{equation}
  \forall\, v \in \Omega,\;\forall\, \mathbf{u} \in \mathbb{S}^2:\quad
  \exists\, s \in \mathcal{S}:\quad
  \mathbf{u}^\top \frac{s - v}{\|s - v\|} = 0.
  \label{eq:tuy}
\end{equation}
This condition is binary: a trajectory either satisfies it globally or it does not, providing no quantitative measure of how much information a constrained trajectory supplies for a given ROI. It further assumes an infinite, noiseless detector and does not account for attenuation, scattered radiation, or finite resolution limits.

\paragraph{Valid source set.}
Not every geometrically available source position provides useful information for a given voxel $v$. We distinguish between sources that are merely geometrically capable of imaging $v$ and those that contribute diagnostically useful information by defining the set of valid source positions $\mathcal{S}_{\text{valid}}(v) \subseteq \mathcal{S}$. A source $s$ is considered valid only if the projection of the ROI containing $v$ lies entirely within the active detector area $\text{Det}$ and if the corresponding projection maintains sufficient signal integrity.
\begin{equation}
  \mathcal{S}_{\text{valid}}(v) =
  \left\{
    s \in \mathcal{S} \;\middle|\;
    \text{proj}_{s}(\text{ROI}) \subseteq \text{Det}
    \;\land\;
    \rho_s(\text{ROI};\alpha) < \eta
  \right\},
  \label{eq:svalid}
\end{equation}
where $\rho_s(\text{ROI};\alpha)$ denotes the fraction of detector pixels inside the projected ROI whose absorption exceeds an attenuation threshold $\alpha$, and $0 < \eta \leq 1$ is the maximum admissible fraction of unusable pixels. Both $\alpha$ and $\eta$ are protocol parameters set during preprocessing.

\paragraph{Discrete voxel-wise approximation.}
Maier et al.~\cite{Maier2015Discrete} and Liu et al.~\cite{Liu2012Completeness} made \cref{eq:tuy} tractable by replacing the continuous requirement over all $\mathbf{u} \in \mathbb{S}^2$ with a check over a finite sampled set $S_U = \{\mu_j\}_{j=1}^z \subset \mathbb{S}^2$. A direction $\mu_j$ is considered \emph{covered} at voxel $v$ if at least one valid source achieves near-orthogonality within angular tolerance $\Delta\gamma$.
\begin{equation}
  \exists\, s \in \mathcal{S}_{\text{valid}}(v):\quad
  \left|\mu_j^\top \frac{s - v}{\|s - v\|}\right| \leq \sin(\Delta\gamma).
  \label{eq:binary_check}
\end{equation}
Herl et al.~\cite{Herl2022X} formalize this as a completeness condition for sets of arbitrary projections and additionally restrict $\mathcal{S}_{\text{valid}}(v)$ to sources with sufficient signal integrity by excluding rays with high mean attenuation~\cite{Herl2020Scanning, Herl2021Task}.

\paragraph{Binary coverage matrix.}
Collecting the per-direction binary checks into a matrix yields the binary coverage matrix $B \in \{0,1\}^{m \times z}$~\cite{Schneider2024IntegerOO}, where $m = |\mathcal{S}|$ is the number of candidate sources and $z$ is the number of sampled directions:
\begin{equation}
  B_{ij} = \mathbb{I}\!\left[\left|\mu_j^\top \frac{s_i - v}{\|s_i - v\|}\right| \leq \sin(\Delta\gamma)\right]
  \cdot \mathbb{I}\!\left[s_i \in \mathcal{S}_{\text{valid}}(v)\right].
  \label{eq:binary_matrix}
\end{equation}
An entry $B_{ij} = 1$ indicates that source $s_i$ covers direction $\mu_j$ within the angular tolerance, subject to validity. This matrix is the discrete counterpart of the classical binary Tuy condition and is used directly by the binary greedy baseline and the binary MILP formulation described in \cref{sec:algorithms}.

\paragraph{Resolution-driven angular tolerance.}
The threshold $\Delta\gamma$ in \cref{eq:binary_check,eq:binary_matrix} is fixed by Nyquist-based resolution requirements~\cite{buzug2011einfuhrung, Shannon1949Communication}. To resolve features of minimum size $f_{\min}$ within an ROI of radius $r$, Fourier-space analysis yields the bound
\begin{equation}
  \Delta\gamma = \frac{f_{\min}}{2r}.
  \label{eq:nyquist_gap}
\end{equation}
This bound simultaneously fixes the angular tolerance in the coverage check and the required sampling density of $S_U$: directions $\mu_j$ must be spaced no coarser than $\Delta\gamma$ so that no direction goes unchecked.

\paragraph{Resolution-driven sphere sampling.}\label{density}
We seek a quasi-uniform discretization of the plane-normal directions on the unit sphere, $S_U = \{\mu_j\}_{j=1}^z \subset \mathbb{S}^2$, with density driven by \cref{eq:nyquist_gap}. To determine the required number of directions $z$, we model the sphere coverage problem using spherical caps. To cover the sphere with caps of half-angle $\theta = \Delta\gamma = f_{\min}/(2r)$, the solid angle of a single cap is
\begin{equation}
  A_\theta = 2\pi(1 - \cos\theta).
\end{equation}
The number of caps required to cover the full sphere ($4\pi$ steradians) is approximated by
\begin{equation}
  z \;\approx\; \frac{4\pi}{A_\theta} \approx \frac{2}{1-\cos\theta}.
\end{equation}
For high-resolution scenarios where $\theta \ll 1$, the Taylor approximation $\cos\theta \approx 1 - \theta^2/2$ gives
\begin{equation}
  z \approx \frac{4}{\theta^2} = \frac{16r^2}{f_{\min}^2}.\label{resolution}
\end{equation}
In practice, we set $z = \lceil 16r^2 / f_{\min}^2 \rceil$ and sample the sphere using a Fibonacci lattice~\cite{Gonzalez2009Measurement}, which yields a numerically stable, quasi-uniform distribution free of the pole singularities of latitude-longitude grids. The spherical-cap argument is an order-of-magnitude estimate; exact coverings require overlap, and the statement that coverage within tolerance $\Delta\gamma$ implies resolvability of $f_{\min}$ is a Nyquist-motivated design guideline rather than a strict guarantee under all imaging conditions.

\subsection{Soft Near-Orthogonality Score}\label{sec:soft_score}

The binary matrix $B$ is the formal starting point, but its hard $0/1$ entries create a flat optimization landscape: all trajectories that achieve the same number of covered directions look identical to any solver. To enable robust discrete optimization, we replace each binary entry with a continuous score that measures the \emph{degree} of near-orthogonality, yielding the soft coverage matrix $A \in [0,1]^{m \times z}$ as a smooth relaxation of $B$.

For a voxel $v$, let $d_i = (s_i - v)/\|s_i - v\|$ denote the normalized ray direction from $v$ to source $s_i$, and let $\tau = \sin(\Delta\gamma) = \sin(f_{\min}/(2r))$ be the dot-product tolerance corresponding to the angular bound. The soft score is defined as
\begin{equation}
  A_{ij} = \max\!\left(0,\frac{\tau - \lvert \mu_j^\top d_i\rvert}{\tau}\right) \cdot \mathbb{I}[s_i \in \mathcal{S}_{\text{valid}}(v)].
  \label{eq:cosine_soft_score}
\end{equation}
The absolute value $\lvert \mu_j^\top d_i\rvert$ accounts for the sign ambiguity of plane normals since $\mu$ and $-\mu$ describe the same Radon plane. The indicator $\mathbb{I}[s_i \in \mathcal{S}_{\text{valid}}(v)]$ enforces the same validity constraints as in $B$: any source failing the detector intersection or attenuation test receives a score of zero for all directions, regardless of geometric alignment.

The score decays linearly from $A_{ij} = 1$ at perfect near-orthogonality ($|\mu_j^\top d_i| = 0$) to $A_{ij} = 0$ at the binary threshold ($|\mu_j^\top d_i| = \tau$), and is zero beyond. The support of $A$ is contained in the support of $B$: $A_{ij} > 0$ implies $B_{ij} = 1$. The converse holds in the strict interior, $|\mu_j^\top d_i| < \tau$; at the boundary $|\mu_j^\top d_i| = \tau$ exactly, $B_{ij} = 1$ but $A_{ij} = 0$. Within the support, $A_{ij}$ provides a graded measure of how centrally the source covers the direction, rewarding sources that lie close to perfect orthogonality over those that barely pass the binary threshold. By operating on dot products rather than angles, the formulation avoids computationally expensive inverse trigonometric operations during large-scale matrix construction.

When scores from multiple selected sources are summed for a given direction $\mu_j$, partial contributions accumulate and can jointly push a direction to full coverage even if no single view achieves $A_{ij} = 1$. Capping the sum at unity prevents over-counting fully covered directions, yielding a saturated directional-coverage objective that rewards both high individual view quality and complementarity across the selected set, a property that binary coverage cannot capture since it is insensitive to the margin by which a direction is hit.

\subsection{Effective Spatial Resolution as Validation Metric}\label{sec:validation_metric}

The soft score and binary matrix serve as optimization objectives; neither directly reports what trajectory quality means in physical terms. The \emph{Effective Spatial Resolution} (ESR) fills this role as a physics-based validation metric: it maps directional sampling gaps in an optimized trajectory to an interpretable spatial resolution limit, independent of which selection method was used. ESR is not an optimization objective; it is evaluated after selection and provides a principled bridge between the geometric selection stage and the image domain without requiring reconstruction.

The ESR is built on the same valid source set $\mathcal{S}_{\text{valid}}(v)$ defined in \cref{eq:svalid} and the same direction grid $S_U$ derived in \cref{resolution}. Let $\mathcal{I}\subseteq\mathcal{S}$ denote the selected source positions. To avoid an undefined minimum when no selected source is valid for a voxel, we first define the per-direction angular gap as
\begin{equation}
  \delta_j(v;\mathcal{I}) =
  \begin{cases}
    \displaystyle \min_{s \in \mathcal{S}_{\text{valid}}(v) \cap \mathcal{I}}
    \arcsin \left( \left| \frac{s-v}{\|s-v\|} \cdot \mu_j \right| \right),
    & \mathcal{S}_{\text{valid}}(v) \cap \mathcal{I} \neq \emptyset, \\[1.2em]
    \displaystyle \frac{\pi}{2},
    & \mathcal{S}_{\text{valid}}(v) \cap \mathcal{I} = \emptyset .
  \end{cases}
  \label{eq:directional_gap}
\end{equation}
The second branch assigns the maximum possible unoriented angular miss and flags the voxel as completely unsupported by the selected valid views. In the experiments below, ROIs for which no candidate source is valid under the chosen detector and attenuation thresholds are excluded before optimization, since no trajectory can provide usable information for such an ROI. For each direction $\mu_j \in S_U$, the mean angular gap is then
\begin{equation}
  \bar\varphi(v) = \frac{1}{z} \sum_{j=1}^z \delta_j(v;\mathcal{I}).
\end{equation}
We employ the arithmetic mean $\bar\varphi(v)$ rather than the maximum gap to provide a robust statistical summary that is less sensitive to isolated outliers in the sampling pattern.

To interpret this angular error physically, we invert the Nyquist bound from \cref{eq:nyquist_gap}~\cite{Tuy1983Inversion, Smith1985Image}. An average angular gap of $\bar\varphi(v)$ corresponds to an estimated smallest resolvable feature size of
\begin{equation}
  \hat{f}_{\text{res}}(v) \approx 2r \cdot \bar\varphi(v).
  \label{eq:effective_resolution}
\end{equation}
A result of $\hat{f}_{\text{res}} = 1.0\,\text{mm}$ implies that features smaller than 1.0\,mm are likely to suffer aliasing artifacts due to insufficient angular sampling frequency, while the trajectory remains adequate for coarser structures.

While the mean ESR is a useful average-case summary, it can underestimate streaking risk when the trajectory contains rare but large directional gaps. To capture more conservative, worst-case behavior, we additionally define a directional quantile error
\begin{equation}
  \varphi_{p}(v) = Q_{p}\!\left(\left\{\delta_j(v;\mathcal{I})\right\}_{j=1}^{z}\right),
\end{equation}
where $Q_p(\cdot)$ denotes the $p$-quantile (e.g., $p=0.95$). This yields a quantile-based estimate $\hat{f}_{\text{res},p}(v) \approx 2r \cdot \varphi_p(v)$, providing a practical near-worst-case bound while remaining robust to isolated outliers. In our evaluation, we report both the mean ESR and a quantile ESR to distinguish average performance from directional-coverage tail risk.

To summarize the ROI as a single conservative number, we further aggregate the voxel-wise estimates using a voxel quantile (e.g., $Q_{0.90}$ or $Q_{0.95}$ over $v\in\Omega$) in addition to the voxel mean. This separates overall performance from localized ``weak spots'' within the ROI.

%% file: content/projection_selection_methods.tex

\section{Projection Selection for Resolution-Aware Radon Sampling}\label{sec:algorithms}

With the binary coverage matrix $B$ and its soft relaxation $A$ established in \cref{sec:completeness}, trajectory design reduces to a combinatorial selection problem: from a finite set of candidate source positions, choose at most $k$ views that maximize the saturated coverage objective. \cref{sec:matrix_construction} describes how $A$ is assembled for a given geometry by combining the resolution-driven near-orthogonality score with detector-visibility and attenuation-validity checks, decoupling the expensive physics evaluation from the subsequent combinatorial search. \cref{sec:greedy_algo} presents a marginal-gain greedy algorithm with a provable $(1-1/e)$ approximation guarantee. \cref{sec:milp} formulates an exact MILP that certifies optimality bounds via branch-and-cut.

\subsection{Construction of the Coverage Matrices}\label{sec:matrix_construction}

Both the binary coverage matrix $B \in \{0,1\}^{m \times z}$ and its soft relaxation $A \in [0,1]^{m \times z}$, formally defined in \cref{sec:completeness}, are computed from the same candidate geometry. Let $\mathcal{S} = \{s_1, \dots, s_m\} \subset \mathbb{R}^3$ denote the discrete set of candidate X-ray source positions, generated by uniform spherical sampling or constrained to a specific kinematic path. Let $S_U = \{\mu_1, \dots, \mu_z\} \subset \mathbb{S}^2$ be the direction grid whose cardinality $z$ is fixed by the resolution-driven density from Section~\ref{density}. Both matrices have dimensions $m \times z$; each entry $(i,j)$ characterises the utility of source $s_i$ for sampling Radon plane normal $\mu_j$ at the ROI center $c$. The voxel-wise formulation in \cref{sec:completeness} motivates the model; the implementation and all experiments instantiate it at the ROI center, with multi-ROI extensions handled by separate center or cluster-center matrices.

\paragraph{Shared validity preprocessing.}
The first two construction steps are identical for both matrices and operationally realise the valid source set $\mathcal{S}_\text{valid}(v)$ from \cref{eq:svalid}.

First, a geometric validity check is performed. For each candidate $s_i$, the principal ray direction $d_i = (s_i - c)/\|s_i - c\|$ is computed and the projection of the ROI centred at $c$ onto the detector is verified. Sources whose projected ROI falls outside the active detector area are geometrically invalid and receive zero entries for all directions in both matrices.

Second, attenuation filtering enforces signal quality. For each geometrically valid source, the projected ROI patch is extracted in detector space and converted to an absorption map. A source remains admissible only if fewer than a fraction $\eta$ of its ROI pixels exceed the attenuation threshold $\alpha$; otherwise it is excluded from $\mathcal{S}_\text{valid}(v)$ and its entire row is set to zero. This fraction-based criterion is less brittle than a single-pixel worst-case test and better reflects whether the projected ROI retains sufficient usable signal.

\paragraph{Matrix population.}
Given the validity-filtered candidates, the two matrix variants are populated in the third step. The binary matrix $B$ applies the hard threshold from \cref{eq:binary_matrix}: $B_{ij} = 1$ if $|\mu_j^\top d_i| \le \sin(\Delta\gamma)$, and $0$ otherwise. The soft matrix $A$ replaces this indicator with the linear decay from \cref{eq:cosine_soft_score}: $A_{ij} = \max(0,\, (\tau - |\mu_j^\top d_i|)/\tau)$, where $\tau = \sin(\Delta\gamma)$. Both scores are computed in the dot-product domain to avoid inverse trigonometric operations at scale. The support of $A$ is contained in the support of $B$: $A_{ij} > 0$ implies $B_{ij}=1$. Conversely, if $B_{ij}=1$ and $|\mu_j^\top d_i|<\tau$, then $A_{ij}>0$; at the boundary $|\mu_j^\top d_i|=\tau$, the binary entry remains one while the soft score is zero. Thus $A$ refines the interior of the binary acceptance band by encoding the degree of near-orthogonality.

Both matrices are typically sparse, as each source covers only a small spherical cap of directions. This sparsity is exploited in the subsequent optimization algorithms to maintain computational tractability even for high-resolution sampling schemes involving thousands of directions.

\subsection{Greedy Projection Selection with Continuous Coverage}\label{sec:greedy_algo}

The greedy algorithm is the primary practical solver for the view-selection problem. It iteratively constructs a trajectory of at most $k$ views by selecting at each step the candidate source that maximises the marginal gain in saturated coverage.

The algorithm maintains a state vector $\boldsymbol{\gamma} \in [0,1]^z$, initialised to zero, whose $j$-th entry records the accumulated coverage of direction $\mu_j$ by the views selected so far. At each iteration the marginal gain of candidate $s_i$ is
\[
  \Delta_i = \mathbf{1}^\top \bigl(\min(\mathbf{1},\, \boldsymbol{\gamma} + A_{i,:}) - \boldsymbol{\gamma}\bigr),
\]
where $A_{i,:} \in [0,1]^z$ is the $i$-th row of the soft coverage matrix. The clamp $\min(\mathbf{1}, \cdot)$ implements directional saturation: once a direction reaches full coverage, any further contribution from $s_i$ to that direction adds nothing to $\Delta_i$. The candidate $i^* = \arg\max_i \Delta_i$ is added to the solution and the state is updated to $\boldsymbol{\gamma} \leftarrow \min(\mathbf{1},\, \boldsymbol{\gamma} + A_{i^*,:})$. Selection is therefore driven by complementarity rather than standalone coverage: a source that covers directions already well-served by earlier selections receives low marginal gain, while a source with moderate standalone score is preferred if it covers poorly-supported directions. This process repeats until $k$ views are selected or no remaining candidate yields positive marginal gain. In the reported experiments the selected budgets remain in the positive-gain regime, so the practical trajectories use the requested number of views.

We additionally evaluate a \emph{binary Greedy} variant that applies the same marginal-gain rule to the binary coverage matrix $B$ instead of $A$. This isolates the effect of the scoring function: differences between soft Greedy and binary Greedy are attributable to the soft relaxation, while differences between binary Greedy and binary MILP reflect the benefit of global combinatorial optimisation. The greedy solution also serves as an initial feasible incumbent for the corresponding MILP, providing a lower bound from the start of the search.

Each of the $k$ selection steps evaluates $\Delta_i$ for at most $m$ candidates, with each evaluation requiring $\mathcal{O}(z)$ operations, giving a total complexity of $\mathcal{O}(kmz)$.

\paragraph{Approximation guarantee.}
The saturated coverage objective $f_{\mathrm{sat}}$ is monotone, since adding a view to any partial solution cannot decrease coverage, and submodular, since the marginal gain of any source is non-increasing as the solution grows. Both properties follow directly from the per-direction saturation structure: each term $\min(1, \sum_{i \in \mathcal{I}} A_{ij})$ is a non-decreasing concave function of the selected set $\mathcal{I}$, and non-negative sums of such functions preserve both properties. Nemhauser et al.~\cite{Nemhauser1978analysis} prove that marginal-gain greedy applied to any monotone submodular function under a cardinality constraint achieves
\[
  f_{\mathrm{sat}}(\mathcal{I}_{\mathrm{greedy}}) \;\ge\; \Bigl(1 - \tfrac{1}{e}\Bigr) \cdot \mathrm{OPT} \;\approx\; 0.632 \cdot \mathrm{OPT},
\]
where $\mathrm{OPT} = \max_{|\mathcal{I}| \le k} f_{\mathrm{sat}}(\mathcal{I})$. This bound is tight in the worst case over all monotone submodular functions; the MILP certificates in \cref{sec:results} show that the true gap is substantially smaller on the CT instances studied here.

\subsection{Mixed-Integer Linear Programming for Global Optimality}\label{sec:milp}

While the greedy heuristic is efficient, it does not guarantee finding the globally optimal set of views. For applications where maximum image quality is critical, or for offline protocol design, we formulate the view selection problem as a MILP. This formulation allows us to leverage exact solvers to explore the combinatorial solution space rigorously.

The role of the MILP in this work is twofold. First, and primarily in this study, it provides a reference model with solver certificates that quantify how close the greedy approximation is to the global optimum. Close agreement between Greedy and soft MILP should therefore be interpreted as evidence about the structure of the proposed saturated coverage objective rather than as evidence that the MILP is unnecessary. Second, it can be used directly as an exact offline planner when certified optimization is required for a given application.

The central modelling challenge is that the saturated per-direction value $\min(1, \sum_i A_{ij} x_i)$ is non-linear and cannot appear directly as a linear objective. We linearise it by introducing, for each direction $\mu_j$, an auxiliary continuous variable $y_j \in [0,1]$ that represents the coverage level the solver is allowed to claim for that direction. A binary variable $x_i \in \{0,1\}$ encodes whether source $s_i$ is selected.

\[
  \begin{aligned}
    \max_{x,y} \quad & \sum_{j=1}^z y_j \\
    \text{s.t.} \quad
    & y_j \le \sum_{i=1}^m A_{ij} \, x_i, \quad &&\forall j \in \{1,\dots,z\} \\
    & \sum_{i=1}^m x_i \le k \\
    & x_i \in \{0,1\}, \quad y_j \in [0,1]
  \end{aligned}
  \tag{MILP-Basic}
  \label{eq:milp_basic}
\]

The inequality $y_j \le \sum_i A_{ij} x_i$ prevents the model from claiming more coverage for direction $\mu_j$ than the selected views actually provide. Because the objective maximises $\sum_j y_j$, the solver pushes each $y_j$ to its largest feasible value. The domain constraint $y_j \le 1$ then acts as the saturation cap: when the selected views provide a combined score above one, $y_j$ is set to one; when they fall short, $y_j$ is limited by their total contribution. At optimality the achieved coverage for each direction is therefore
\[
  y_j^* = \min\!\left(1,\, \sum_{i \in \mathcal{I}} A_{ij}\right),
\]
and the objective value equals $f_{\mathrm{sat}}(\mathcal{I}) = \mathbf{1}^\top \min\!\left(\mathbf{1}, \sum_{i \in \mathcal{I}} A_{i,:}\right)$, the saturated coverage defined in \cref{sec:soft_score}.

For reporting, we normalise by the total number of directions $z$ to obtain the \emph{normalized saturated coverage}
\[
  \bar{f}_{\mathrm{sat}}(\mathcal{I}) = \frac{1}{z} \sum_{j=1}^{z} y_j \;\in [0,1],
\]
which can be directly compared with the binary Tuy score and the Soft-Tuy score on a common scale. In addition, the implementation reports
\[
  \operatorname{SoftTuy}(\mathcal{I}) = \frac{1}{z} \sum_{j=1}^{z} \max_{i \in \mathcal{I}} A_{ij},
\]
which averages the best coverage score provided by any single selected view for each direction, without accumulating contributions from multiple views. SoftTuy is a useful cross-method diagnostic: unlike $\bar{f}_{\mathrm{sat}}$, it does not reward complementary multi-view coverage and is therefore lower or equal. The following remark makes this ordering precise.

\begin{remark}
For any selection $\mathcal{I}$ and non-negative coverage matrix $A$,
$\bar{f}_{\mathrm{sat}}(\mathcal{I}) \ge \operatorname{SoftTuy}(\mathcal{I})$.
\end{remark}
\begin{proof}
Fix any direction $\mu_j$ and let $m_j = \max_{i\in\mathcal{I}} A_{ij}$ and
$s_j = \sum_{i\in\mathcal{I}} A_{ij}$.  Because all entries are non-negative,
$s_j \ge m_j$.  Since $A_{ij}\in[0,1]$ implies $m_j \le 1$, applying
$\min(1,\cdot)$ to both sides gives
$c^{\mathrm{sat}}_j = \min(1,s_j) \ge \min(1,m_j) = m_j = c^{\mathrm{soft}}_j$.
Averaging over all $j$ gives $\bar{f}_{\mathrm{sat}}(\mathcal{I}) \ge
\operatorname{SoftTuy}(\mathcal{I})$.  Equality holds for the averaged metrics
if and only if, for every sampled direction, either $s_j=m_j$ or $m_j=1$.
Equivalently, no direction receives additional positive unsaturated support
beyond its best selected view.
\end{proof}

\paragraph{Three coverage readouts.}
The three scalar metrics reported throughout this work differ only in how they aggregate per-direction scores across a selected view set $\mathcal{I}$. For a fixed direction $\mu_j$, let the selected views induce scores $\{A_{ij}\}_{i \in \mathcal{I}}$. Then
\[
  c^{\mathrm{bin}}_j = \mathbb{I}\!\left[\max_{i \in \mathcal{I}} A_{ij} > 0\right], \qquad
  c^{\mathrm{soft}}_j = \max_{i \in \mathcal{I}} A_{ij}, \qquad
  c^{\mathrm{sat}}_j = \min\!\left(1, \sum_{i \in \mathcal{I}} A_{ij}\right).
\]
Binary Tuy retains only hit-or-miss information. Soft-Tuy records the best single-view support per direction. Saturated coverage additionally rewards complementary partial contributions from multiple views, but caps the total at one. The normalized metrics are means of $c^{\mathrm{bin}}_j$, $c^{\mathrm{soft}}_j$, and $c^{\mathrm{sat}}_j$ over all $z$ sampled directions. \cref{fig:metric_aggregation_concept} illustrates all three.

\begin{figure}[t]
  \centering
  \includegraphics[width=\linewidth]{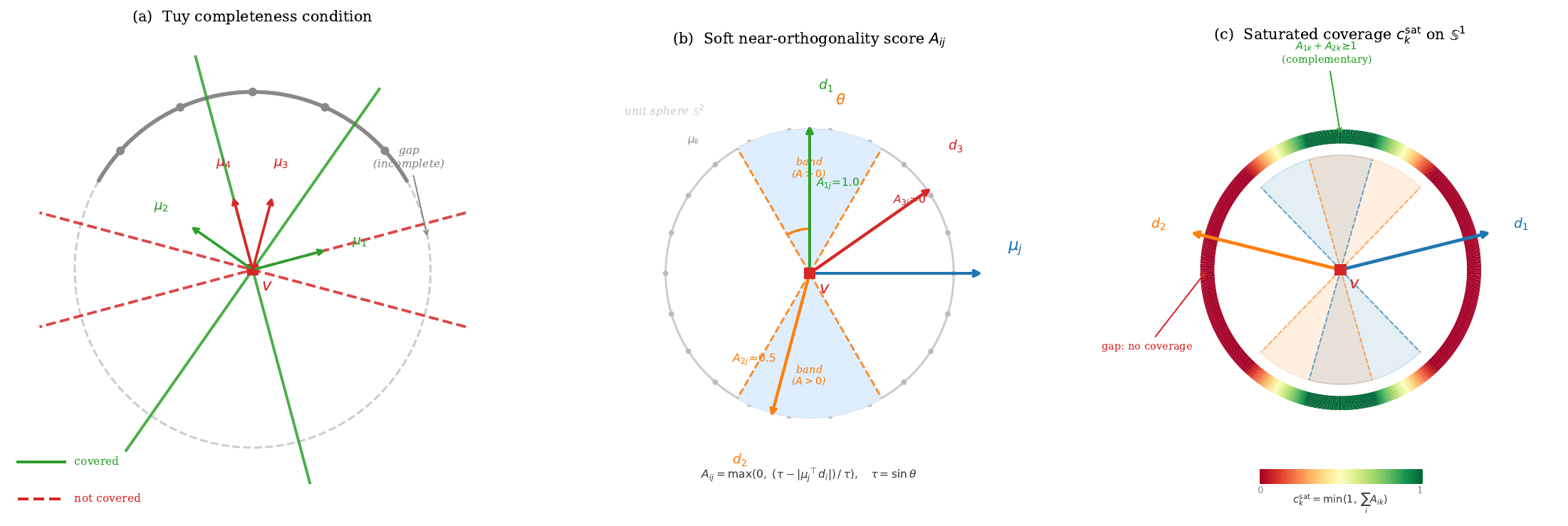}
  \caption{From the classical Tuy condition to soft coverage scores on the unit sphere.
  \textbf{(a)}~The Tuy completeness condition requires that every Radon plane
  through voxel~$v$ intersects the source trajectory.  A plane with normal
  $\mu$ is covered when the arc contains a source position at angle
  $\mu\pm90^\circ$, i.e.\ a source lying on that plane.  The planes $\mu_1$ and
  $\mu_2$ (solid green) are covered because perpendicular source positions
  (${\approx}105^\circ$ and ${\approx}55^\circ$) exist on the arc; the planes $\mu_3$
  and $\mu_4$ (dashed red) are not covered because the required perpendicular
  positions (${\approx}165^\circ$ and ${\approx}195^\circ$) fall in the gap.
  \textbf{(b)}~Discretising the Radon plane normals onto the unit sphere
  $\mathbb{S}^2$ yields sampled directions~$\mu_j$ (grey dots).  A
  selected source contributes a ray direction~$d_i$ from voxel~$v$; source
  $d_i$ covers direction~$\mu_j$ if $|{\mu_j}^\top d_i|$ is smaller than
  the threshold~$\tau=\sin\theta$, i.e.\ the ray lies within the two
  equatorial bands of~$\mu_j$.  The soft score $A_{ij}$ decays linearly from
  one (perfect orthogonality, $d_1$) to zero at the band boundary ($d_3$).
  \textbf{(c)}~Saturated coverage $c^{\mathrm{sat}}_j = \min(1,\sum_i A_{ij})$
  visualised as a heatmap on the unit circle for two selected views
  $d_1$~(blue) and $d_2$~(orange).  Directions covered by both views
  accumulate complementary support and saturate (green); directions in the
  gap between the two view bands receive no coverage (red).  Partial coverage
  from a single view produces intermediate values (yellow).}%
  \label{fig:metric_aggregation_concept}
\end{figure}

\paragraph{Worst-case objective (max--min coverage).}
The sum objective above optimizes average directional coverage and may admit solutions with rare but large directional gaps if most directions are well-covered. To explicitly penalize such gaps, we introduce a worst-case objective that maximizes the minimum coveredness across all target directions. We add a scalar variable $t \in [0,1]$ representing a guaranteed lower bound on directional coverage, with the additional constraint $t \le y_j$ for all $j$.
\begin{equation}
  \begin{aligned}
    \max_{x,y,t} \quad & t \\
    \text{s.t.} \quad
    & t \le y_j, && \forall j \in \{1,\dots,z\} \\
    & y_j \le \sum_{i=1}^m A_{ij} x_i, && \forall j \in \{1,\dots,z\} \\
    & \sum_{i=1}^m x_i \le k \\
    & x_i \in \{0,1\}, \quad y_j \in [0,1], \quad t \in [0,1]
  \end{aligned}
  \tag{MILP-WorstCase}
  \label{eq:milp_worstcase}
\end{equation}
This formulation is more conservative and directly targets the elimination of large uncovered regions in normal space by improving the least-covered direction. In the implementation, a robust variant is used by default: directions that are globally uncoverable under the current validity mask (all-zero coverage columns) can be excluded from the $t \le y_j$ floor constraints, and a secondary tie-break objective maximizes mean coveredness after maximizing $t$.

\subsubsection{Extension to Multiple ROIs}

The MILP formulation extends naturally to trajectories that must serve several disjoint ROIs simultaneously. This is particularly useful when targets are spatially close, such as adjacent vertebrae or neighboring defects in a casting, because their angular requirements partially overlap and a joint solution can find views that contribute to multiple ROIs at once. For targets so widely separated that their coverage requirements are geometrically incompatible, separate acquisitions will generally outperform any joint trajectory.

\paragraph{Cluster construction.}
Given $Q$ ROI centres $\{v^1, \dots, v^Q\}$, geometrically close ROIs are first grouped into $C \le Q$ clusters to control the size of the resulting program. ROIs whose centres lie within a distance threshold $d_\mathrm{fuse}$ are assigned to the same cluster. Each cluster $c$ with member set $\mathcal{C}_c \subseteq \{1,\ldots,Q\}$ is represented by a centroid $\bar{v}_c = |\mathcal{C}_c|^{-1}\sum_{q \in \mathcal{C}_c} v^q$ and an inflated effective radius
\[
  r_c = r + \max_{q \in \mathcal{C}_c} \|v^q - \bar{v}_c\|,
\]
where $r$ is the nominal ROI radius. The coverage matrix $A^c \in [0,1]^{m \times z_c}$ and its angular tolerance are derived from $\bar{v}_c$ and $r_c$ exactly as in the single-ROI case, ensuring the cluster conservatively covers every member. The validity mask for the cluster requires a projection to be valid for all members simultaneously. Without fusion, $C = Q$ and each ROI forms its own singleton cluster.

\paragraph{Joint optimisation program.}
With per-cluster weights $w_c > 0$ and shared binary selection variables $x_i$, the joint program reads
\[
  \begin{aligned}
    \max_{x,y} \quad & \sum_{c=1}^C w_c \sum_{j=1}^{z_c} y_j^c,\\
    \text{s.t.} \quad & y_j^c \le \sum_{i=1}^m A_{ij}^c \, x_i, \quad &&c=1,\ldots,C,\;
j=1,\ldots,z_c,\\
    & \sum_{i=1}^m x_i \le k,\\
    & x_i \in \{0,1\}, \quad y_j^c \in [0,1],
  \end{aligned}
  \tag{MILP-Multi-ROI}
  \label{eq:milp_multiroi}
\]
where $z_c$ is the number of sampled directions for cluster $c$. The selection variables $x_i$ are shared across all clusters, so the solver identifies views that contribute to multiple targets simultaneously. The number of constraints grows linearly with $\sum_c z_c$; clustering directly controls this scaling at the cost of a more conservative coverage model.

\paragraph{Weighting strategies.}
The weights $w_c$ determine how the objective balances coverage across clusters. \emph{Distance-weighted} fusion sets $w_c = |\mathcal{C}_c|$, so each cluster's contribution scales with the number of ROIs it represents; every individual ROI thereby carries equal weight in the objective regardless of how the grouping turned out. \emph{Uniform} fusion merges all $Q$ ROIs into a single cluster ($C=1$, $w_1=1$), treating the entire target set as one inflated region without distinguishing between tightly and loosely grouped members.

\subsection{Optimization Strategy: Branch-and-Cut}\label{sec:branch_and_cut}

The projection selection problem formulated in~\cref{eq:milp_basic} is a MILP. Unlike convex programs where every local optimum is also global, general MILPs are NP-hard~\cite{Garey1979ComputersAI} and no polynomial-time algorithm is known for arbitrary instances. To certify globally optimal trajectories for moderate problem sizes, we employ branch-and-cut, which integrates combinatorial tree search with polyhedral cutting-plane methods~\cite{Schrijver1998Theory, Huang2021BranchAB}.

The algorithm begins with the LP relaxation, obtained by replacing the integrality constraints $x_i \in \{0,1\}$ with $x_i \in [0,1]$. This relaxation is solvable in polynomial time and its optimal value is a rigorous upper bound on the best integer-feasible objective. The search then proceeds by branching: fractional variables are recursively fixed to $0$ or $1$, growing a binary search tree. Subtrees whose LP upper bound does not exceed the best integer solution found so far are pruned. Cutting planes are added dynamically to tighten the LP relaxation and reduce tree size, approximating the convex hull of the integer-feasible set~\cite{Schrijver1998Theory}.

A critical advantage of this exact approach over heuristic baselines (such as greedy selection) is the provision of a certified optimality gap, defined as
\begin{equation}
  \text{Gap} = \frac{\text{Upper Bound} - \text{Best Feasible Objective}}{\text{Upper Bound}}.
\end{equation}
This metric serves as a rigorous stopping criterion, allowing the trade-off between computation time and solution quality to be explicitly managed. In our implementation, all MILP instances are solved using the Gurobi optimizer~\cite{gurobi}.

%% file: content/experiments.tex

\section{Experimental Setup}\label{sec:experiments}

All experiments were carried out on a simulated cone-beam CT data set of an aluminium component drawn from the ABC dataset~\cite{Koch2019ABC}. Projections were generated by forward-projecting the volume through DiffCT's cone-beam projector~\cite{diffct2025} under a monochromatic Beer-Lambert model. The 800 candidate source positions form a quasi-uniform Fibonacci sphere, providing isotropic angular coverage as the candidate pool for all selection methods. Six target regions of interest were used throughout all optimization, validity, and reconstruction evaluations. All acquisition, object, and reconstruction parameters are listed in \cref{tab:sim_params}.

\begin{table}[t]
  \centering
  \caption{Simulation and reconstruction parameters used throughout all experiments.}
  \label{tab:sim_params}
  \begin{tabular}{ll}
    \hline
    \multicolumn{2}{l}{\textit{Scan geometry}} \\
    \hline
    Source-to-isocenter distance (SID) & $2000\,\mathrm{mm}$ \\
    Source-to-detector distance (SDD)  & $4000\,\mathrm{mm}$  \\
    Detector                           & $256\times256\,\mathrm{px}$, $0.9\,\mathrm{mm}$ pixel pitch \\
    Candidate views                    & 800, Fibonacci sphere \\
    \hline
    \multicolumn{2}{l}{\textit{Object}} \\
    \hline
    Material                           & Aluminium \\
    Attenuation coefficient            & $\mu \approx 0.416\,\mathrm{mm}^{-1}$ (${\approx}36\,\mathrm{keV}$) \\
    Forward projection                 & Monochromatic Beer-Lambert (DiffCT~\cite{diffct2025}) \\
    \hline
    \multicolumn{2}{l}{\textit{Reconstruction (ASD-POCS~\cite{Sidky2008})}} \\
    \hline
    Voxel grid                         & $384^3$, $0.3\,\mathrm{mm}$ isotropic \\
    Outer iterations                   & 10 \\
    SART sweeps per outer iteration    & 2 \\
    Projection subsets                 & 8 \\
    SART relaxation                    & 0.8 (normalized) \\
    TV steps per outer iteration       & 1 \\
    TV parameters                      & $\alpha = 0.005$, $\epsilon = 0.002$ \\
    \hline
    \multicolumn{2}{l}{\textit{Occluders (stress test)}} \\
    \hline
    Type                               & Axis-aligned rectangular plates \\
    Thickness                          & $14\,\mathrm{mm}$ \\
    Attenuation coefficient            & $\mu = 2.5\,\mathrm{mm}^{-1}$ (${\approx}6\times$ aluminium) \\
    \hline
  \end{tabular}
\end{table}

To study method robustness under varying attenuation conditions, we constructed a controlled occlusion stress test on the same simulated object. Starting from the original volume, we generated four conditions (\texttt{none}, \texttt{mild}, \texttt{moderate}, \texttt{severe}; \cref{fig:occlusion_stress_asd_pocs_800_views}) by inserting rectangular plate occluders ($\mu = 2.5\,\mathrm{mm}^{-1}$, approximately six times the aluminium attenuation) of 14\,mm thickness outside the estimated object bounding box. The reconstruction target, projection order, number of views, source-detector poses, detector settings, and full DiffCT forward-projection pipeline were kept fixed; only the volume content was modified. This isolates attenuation-induced visibility degradation from trajectory or geometry changes.

\begin{figure}[t]
  \centering
  \includegraphics[width=\linewidth]{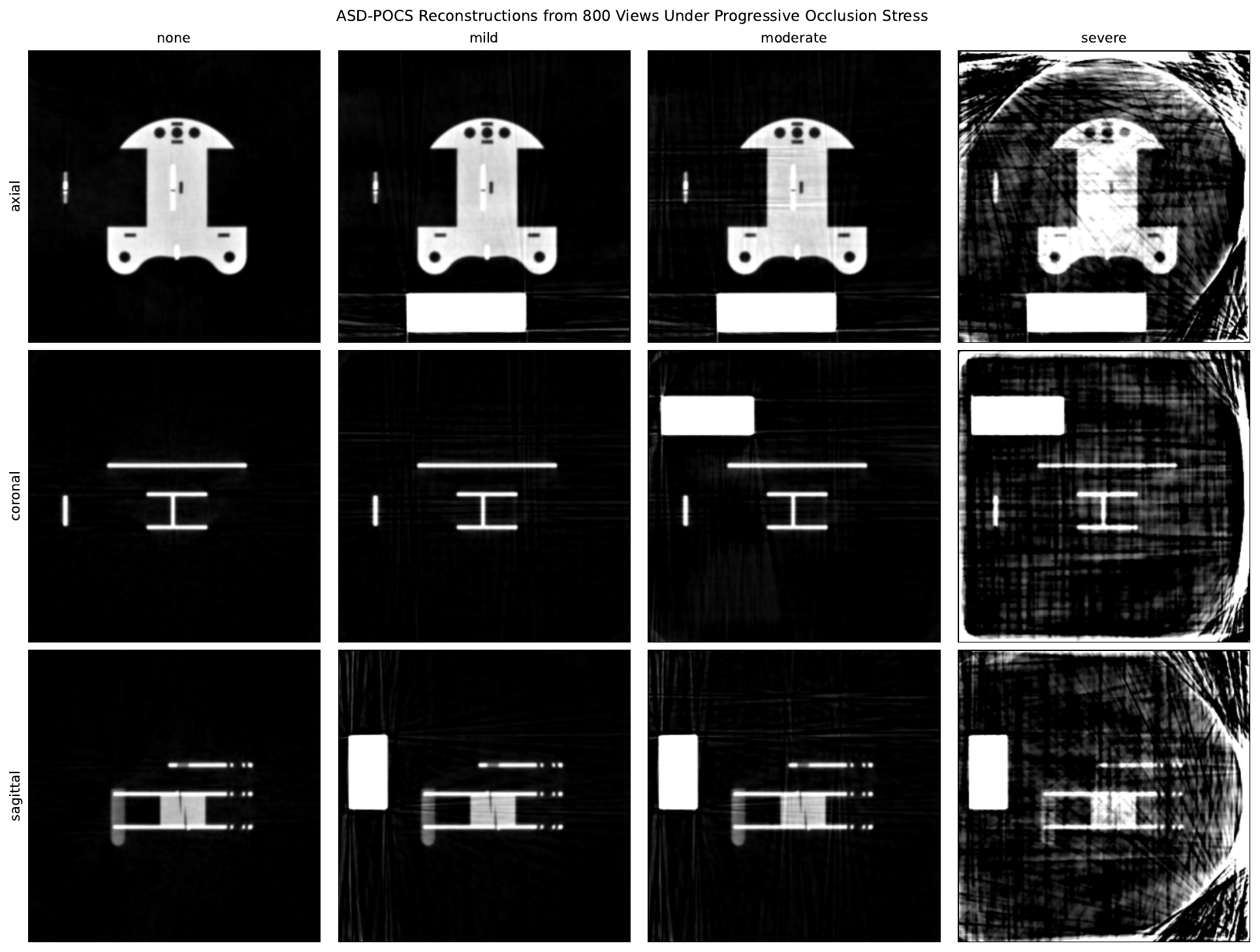}
  \caption{ASD-POCS reconstructions from all 800 calibrated views for the controlled occlusion stress test. Columns show the unoccluded, mild, moderate, and severe occlusion conditions; rows show matched axial, coronal, and sagittal centre sections.}%
  \label{fig:occlusion_stress_asd_pocs_800_views}
\end{figure}

Validity was evaluated in detector space using the same ROI projection framework as in the optimization pipeline, with $\eta=0.25$. The threshold $\alpha$ was estimated once as the $95^\mathrm{th}$ percentile of all ROI-pixel absorption values over the unoccluded baseline and then held fixed across all occlusion levels. Under this protocol, the realized mean ROI validities were $96.83\%$ (\texttt{none}), $77.71\%$ (\texttt{mild}), $68.96\%$ (\texttt{moderate}), and $56.19\%$ (\texttt{severe}), corresponding to a mean of $25$, $178$, $248$, and $350$ excluded projections per ROI, a monotone increase in visibility degradation under an otherwise unchanged acquisition setup. The full set of study families run on this data is listed in \cref{tab:full_suite_overview}.

\begin{table}[t]
  \centering
  \caption{Overview of all study families. Each row names a family, the number of distinct experimental scenarios it comprises including all four occlusion levels and all evaluated methods.}%
  \label{tab:full_suite_overview}
  \resizebox{\linewidth}{!}{\input{content/assets/experiments_20260410/tables/table_suite_overview_full.tex}}
\end{table}

All study families listed in \cref{tab:full_suite_overview} were evaluated across all four occlusion levels; the scenario counts in the table reflect this together with the evaluated methods. Binary Greedy is additionally included as a paired diagnostic: it isolates the contribution of the binary completeness model before exact optimization and supplies the feasible incumbent used to warm-start the binary MILP. The study families cover four questions. The main single-ROI budget comparison characterises how coverage objective and selection quality relate as the projection budget varies. A matched reconstruction evaluation transfers those results to the image domain. A multi-target study extends the framework to simultaneous optimization over multiple ROIs. Controlled ablations isolate sensitivity to resolution demand, attenuation-based validity constraints, directional sampling density, and the worst-case objective.

\paragraph{Code, data, and ethics statement.}
The experiments use simulated projections of an industrial CAD component from the ABC dataset; no patient data, human-subject data, consent, or ethics approval is involved. To support reproducibility, the code release accompanying this work will include the scripts used to generate candidate views, ROI definitions, occlusion volumes, DiffCT forward projections, Gurobi selection instances and solver settings, reconstruction parameters, and the tables and figures reported here. Large generated projection and reconstruction arrays will be accompanied by generation scripts and metadata rather than treated as manually curated input data.

Throughout all study families, selection quality is characterised by four metrics. Binary Tuy measures hard directional completeness. Soft Tuy measures continuous directional support averaged over the sampled directions. Saturated coverage is the quantity directly optimized by Greedy and soft MILP, reported on a normalized scale; binary selections are additionally cross-evaluated on the soft $A$ matrix to enable direct comparison. ESR is the mean ROI value in millimetres. Image reconstruction is performed for study families in which image quality is directly informative; repeated projection subsets are deduplicated before reconstruction. The reference volume is reconstructed once from all 800 views; subset reconstructions use the same fixed parameter set (\cref{tab:sim_params}). We report MSE, PSNR, and three-dimensional SSIM over fixed physical ROI sections.

%% file: content/assets/experiments_20260410/tables/table_suite_overview_full.tex
\begin{tabular}{lrrp{7.3cm}}
\hline
Family & Scenarios & Solver rows & Purpose \\
\hline
k\_vs\_tuy\_completeness & 120 & 240 & main budget sweep for matched single-ROI comparison \\
multi\_roi\_fusion\_ablation & 5 & 10 & fusion strategy comparison for multi-ROI MILP \\
multi\_voi & 9 & 18 & clustered multi-target comparison \\
multi\_voi\_scaling & 12 & 24 & scaling with increasing number of ROIs \\
resolution\_sweep\_roi\_pool & 36 & 72 & f\_min sweep over the fixed cleaned ROI pool \\
sphere\_sampling\_ablation\_roi\_pool & 36 & 72 & directional sampling ablation over the fixed cleaned ROI pool \\
validity\_ablation\_roi\_pool & 36 & 72 & attenuation-validity ablation over the fixed cleaned ROI pool \\
worst\_case\_ablation & 18 & 36 & selected-budget worst-case objective ablation \\
\hline
\end{tabular}

%% file: content/results.tex

\section{Results}\label{sec:results}

All results derive from the simulated cone-beam CT setup described in \cref{sec:experiments}. Within this controlled setting, observed differences between methods reflect completeness modelling, objective formulation, and optimization strategy rather than geometric or detector-coverage artefacts.

\subsection{Main Single-ROI Budget Comparison}\label{sec:main_budget}

Soft and binary completeness objectives optimize fundamentally different quantities; this experiment tests whether that difference translates into a consistent and measurable gap in coverage quality and matched reconstruction as the projection budget increases across the studied range.

\begin{figure}[htbp]
  \centering
  \includegraphics[width=\linewidth]{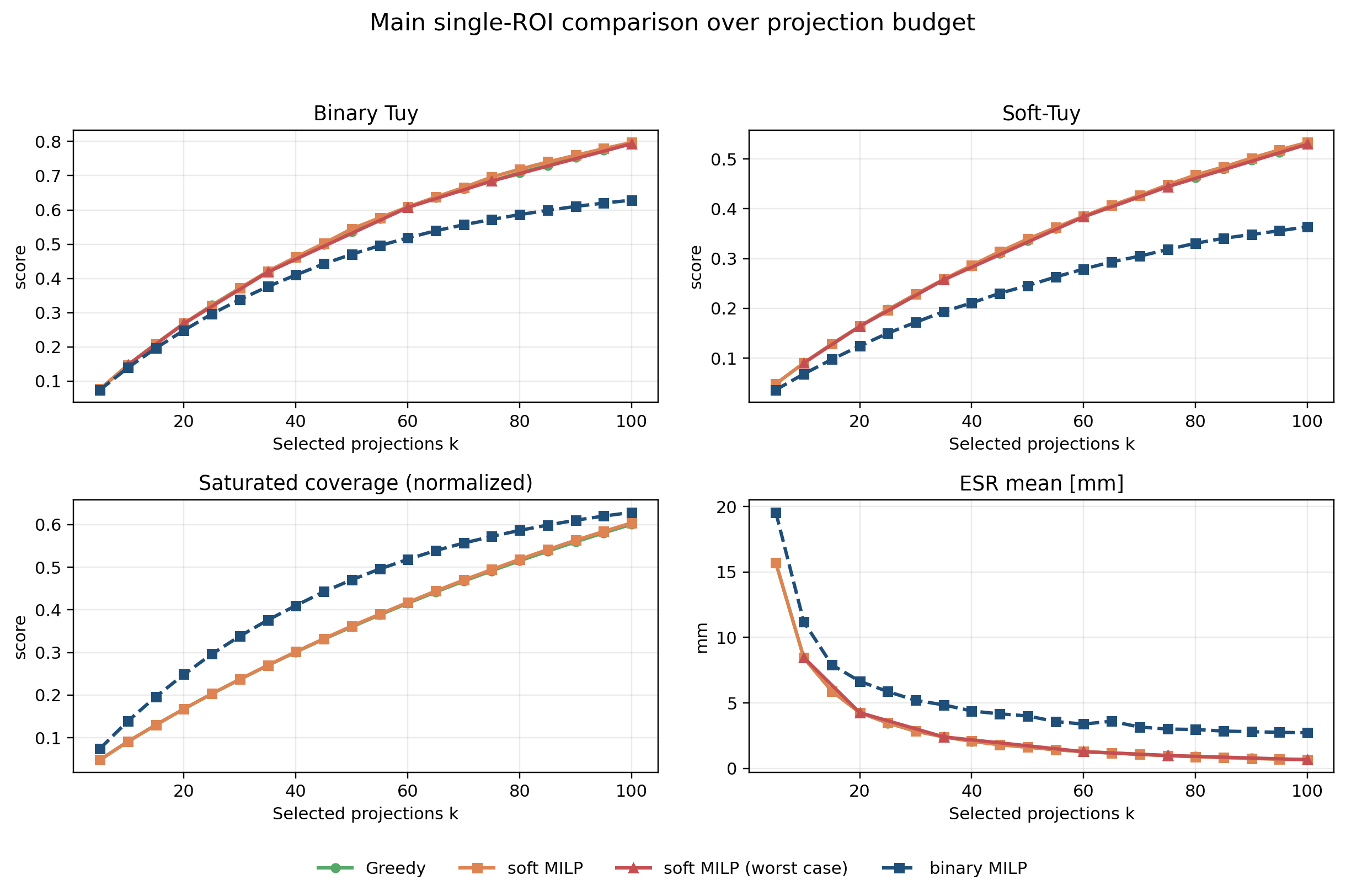}
  \caption{Main single-ROI budget comparison across projection budgets $k \in [5,100]$. Coverage metrics for soft Greedy and soft MILP follow a concave curve, with large gains up to $k \approx 60$ and diminishing returns beyond, consistent with the submodular objective. ESR decreases steeply in the low-budget range and flattens near the Nyquist-adequate coverage threshold. Soft Greedy and soft MILP track each other closely across all panels. The gap between soft methods and binary MILP widens with increasing budget, and binary MILP shows substantially higher point-to-point variance. The saturated coverage trace is omitted for the worst-case formulation as it optimizes a non-standard variant of that objective.}
  \label{fig:k_main_selection_metrics_20260410}
\end{figure}

\cref{fig:k_main_selection_metrics_20260410} and \cref{tab:k_checkpoints_reco_20260410} summarise the results. Two patterns are consistent throughout. First, all soft formulations follow a concave improvement curve, with most coverage gain accumulating before $k=60$ and substantially smaller returns beyond, a direct consequence of the submodular structure of saturated coverage. Soft Greedy and soft MILP track each other closely at every checkpoint, indicating that marginal-gain greedy already extracts nearly all available gain from the soft objective without exact optimization. ESR falls below the target resolution $f_\mathrm{min} = 1\,\mathrm{mm}$ (the Nyquist-motivated threshold derived from the sampling model in \cref{sec:completeness}) only between $k=60$ and $k=100$ for the soft methods, meaning the selection framework operates in the resolution-limited regime for most of the studied budget range. The reconstruction metrics follow the same ordering, with ROI PSNR increasing steadily with budget for all soft formulations. The worst-case soft MILP variant is included throughout as the third soft formulation; its comparative behaviour relative to the mean-objective variants is specific to the objective's minimax structure and is analysed in detail in \cref{sec:ablation}.

\begin{table}[thbp]
  \centering
  \caption{Representative checkpoints from the main single-ROI comparison (unoccluded stage). Selection metrics (Binary Tuy, Soft Tuy, ESR) are mean $\pm$ std over the target configuration pool. Reconstruction metrics (MSE, PSNR, SSIM) are reported without variance. Opt.\ gap is the mean Gurobi-reported MIP gap at termination; soft Greedy has no gap (greedy algorithm) and the worst-case formulation does not yield a standard Gurobi bound.}
  \label{tab:k_checkpoints_reco_20260410}
  \resizebox{\linewidth}{!}{\input{content/assets/experiments_20260410/tables/table_k_checkpoints_reconstruction.tex}}
\end{table}

Second, binary MILP trails all soft formulations at every checkpoint by a wide margin, and its inter-instance variance is substantially larger than that of any soft method. The optimality gap column in \cref{tab:k_checkpoints_reco_20260410} reveals that this gap is not caused by incomplete solver convergence, since binary MILP reaches certified optimal or near-optimal binary solutions in almost all instances, with mean MIP gaps below $1\%$ across all checkpoints. Soft MILP's gap is slightly larger (up to $2.05\%$ at $k=100$, where the five-minute time limit is exhausted on all three instances) but remains within the range certified by the greedy analysis in \cref{tab:greedy_milp_certificates_occlusion}.

A direct numerical comparison of the two MIP gaps would be misleading, since the two solvers optimize over fundamentally different objectives. Soft MILP maximizes a continuous saturated coverage score bounded in $[0,1]$, while binary MILP maximizes an integer count of covered plane normals. The direction of the budget trend further separates the two formulations. Soft MILP's gap grows monotonically with $k$ ($0.99\% \to 1.00\% \to 2.05\%$), reflecting the increasing combinatorial complexity as the branch-and-cut search has to account for more projection candidates. Binary MILP shows no such growth; its gap stays below $1\%$ at every checkpoint regardless of budget, confirming that the binary integer programme is structurally easier to certify within the same time budget. The soft formulation therefore produces solutions that are simultaneously less tightly certified and dramatically better in coverage quality. This demonstrates that the observed performance difference is a consequence of objective formulation. Under the continuous coverage and reconstruction metrics used here, the optimal binary trajectory is weaker because the binary model discards graded directional support. The high inter-instance variance in binary MILP metrics reflects genuine geometric variability across the three evaluated ROIs under the binary model's hard threshold, not solver failure.

The same soft-method advantage holds across all four occlusion stages, as shown in \cref{tab:cross_stage_selection_k100}. Soft Greedy and soft MILP are closely matched at every stage, Binary Tuy decreases only from $0.793$ to $0.753$ from the unoccluded to the severe stage, and mean ESR stays below $0.81\,\mathrm{mm}$ throughout. Binary MILP is substantially more sensitive to occlusion, with mean ESR rising from $2.72 \pm 1.38\,\mathrm{mm}$ at the unoccluded stage to $14.54 \pm 8.43\,\mathrm{mm}$ under severe occlusion, a more-than-fivefold increase, and Binary Tuy variance grows from $\pm 0.087$ at the unoccluded stage to $\pm 0.157$ at moderate, reflecting the diverging sensitivity of different ROIs to the tightening validity constraints rather than solver failures. The optimality gap column in \cref{tab:cross_stage_selection_k100} shows the reverse trend for binary MILP: its gap falls from $0.77\%$ under no occlusion to $0.13\%$ under severe occlusion. As more projections are excluded by the validity constraints, the feasible set shrinks and the binary integer programme becomes easier to certify, even though the resulting trajectory quality collapses.

\begin{table}[thbp]
  \centering
  \caption{Selection metrics at $k=100$ for all four occlusion stages. Soft methods are stable across stages; binary MILP degrades strongly under moderate and severe occlusion. Opt.\ gap is the mean Gurobi-reported MIP gap at termination; its decrease under heavier occlusion reflects fewer valid candidates, not improved solver convergence.}
  \label{tab:cross_stage_selection_k100}
  \resizebox{\linewidth}{!}{\input{content/assets/experiments_20260410/tables/table_cross_stage_selection_k100.tex}}
\end{table}

\subsection{Greedy versus Exact Optimization}

The MILP certificates confirm that the greedy approximation is already near-optimal on the studied instances. Because both soft Greedy and soft MILP optimize the same saturated coverage objective, the exact solver provides the certified reference value against which greedy is assessed rather than changing the qualitative character of the solution. On the unified saturated coverage scale, soft Greedy reaches $0.601 \pm 0.002$ at 100 projections, achieving $99.8\%$ of the soft-MILP objective on average across all paired runs. The paired comparison between binary Greedy and soft Greedy further shows that the binary completeness model, rather than any solver artefact, is responsible for the performance gap. Binary Greedy trails soft Greedy at every checkpoint in all reported metrics. At 100 projections it reaches Binary Tuy $0.626 \pm 0.083$ against $0.793 \pm 0.005$ for soft Greedy, and Soft Tuy $0.367 \pm 0.036$ against $0.530 \pm 0.004$. The soft objective thus produces stronger directional coverage even under the hard binary completeness measure, with lower variance throughout.

The MILP certificates make this relationship more precise (\cref{tab:greedy_milp_certificates_occlusion}). The theoretical guarantee for monotone submodular maximization under a cardinality constraint is the worst-case bound $1-1/e \approx 0.632$, but the observed instances are far from this adversarial regime. Across the four occlusion levels of the main budget study, soft Greedy exactly matched the soft-MILP objective in 72 of 240 paired cases, all certified optimal by the MILP solver. In the remaining cases soft MILP improved the objective only marginally, with a pooled median ratio between soft Greedy and the soft-MILP incumbent of $0.998$ and a minimum of $0.989$. Against the solver's upper bound $U$, the pooled median lower bound $f_G/U$ was $0.989$. The MILP thus does not merely compete with soft Greedy; it quantifies that the marginal-gain approximation is much closer to optimality on these CT instances than the generic worst-case guarantee would imply.

\begin{table}[htbp]
  \centering
  \caption{Instance-specific optimality assessment of marginal-gain soft Greedy against soft MILP in the main budget study. \(f_G\) is the soft Greedy saturated-coverage objective, \(f_I\) is the incumbent soft-MILP objective, and \(U\) is the MILP solver upper bound. Greedy=opt.\ counts cases where soft Greedy matches a certified optimal soft-MILP objective; MILP opt.\ counts certified optimal soft-MILP runs.}
  \label{tab:greedy_milp_certificates_occlusion}
  \resizebox{\linewidth}{!}{\input{content/assets/experiments_20260410/tables/table_greedy_milp_certificates_occlusion.tex}}
\end{table}

The matched reconstruction evaluation confirms the same ordering (\cref{tab:greedy_binary_vs_soft_20260411}). At 100 projections binary Greedy reaches $24.32\,\mathrm{dB}$ in ROI PSNR against $36.34\,\mathrm{dB}$ for soft Greedy, a gap of $12.02\,\mathrm{dB}$ consistent with the binary MILP shortfall in \cref{tab:k_checkpoints_reco_20260410}. The ESR gap is equally large: $2.74 \pm 1.36\,\mathrm{mm}$ for binary Greedy against $0.67 \pm 0.02\,\mathrm{mm}$ for soft Greedy at $k=100$. Because binary Greedy is the marginal-gain result for the binary objective with no solver time limit, this confirms that the performance difference originates in the objective formulation itself rather than in branch-and-cut approximation. The controlled validity ablation in \cref{sec:ablation} shows that the gap has two separable components, a model-structure term and a validity-amplification term that together account for the full binary shortfall. Binary Greedy is therefore not redundant with binary MILP in the analysis; it isolates the model's contribution before exact optimization and supplies the feasible warm-start incumbent for the binary MILP solver.

\begin{table}[htbp]
  \centering
  \caption{Paired comparison between binary Greedy and soft Greedy under identical target and budget conditions. Binary Greedy applies the marginal-gain rule to the legacy binary coverage matrix and supplies the feasible warm-start for the binary MILP.}
  \label{tab:greedy_binary_vs_soft_20260411}
  \resizebox{\linewidth}{!}{\input{content/assets/experiments_20260410/tables/table_greedy_binary_vs_soft.tex}}
\end{table}

\subsection{Reconstruction Transfer}

\cref{fig:k_reconstruction_20260410} reports reconstruction quality across budgets for both full-volume and ROI evaluation. The primary signal is the ROI metric, since the trajectory is optimized for a specific target region; full-volume averages pool the challenging target region with easier surrounding tissue and therefore overstate quality relative to what matters for the clinical target. The ROI ordering mirrors the coverage ordering throughout. Soft Greedy and soft MILP remain closely matched at every checkpoint; ROI PSNR and SSIM rise steeply at low-to-moderate budgets and level off towards $k=100$, consistent with the concave submodular coverage structure. Binary MILP lags all soft formulations by more than $12\,\mathrm{dB}$ in ROI PSNR at the highest budget, a gap fully attributable to the coverage model rather than solver convergence. The full-volume metrics show the same method ordering but with larger absolute values and compressed differences between methods, consistent with the surrounding tissue being easier to reconstruct and less sensitive to the choice of trajectory.

\begin{figure}[htbp]
  \centering
  \includegraphics[width=\linewidth]{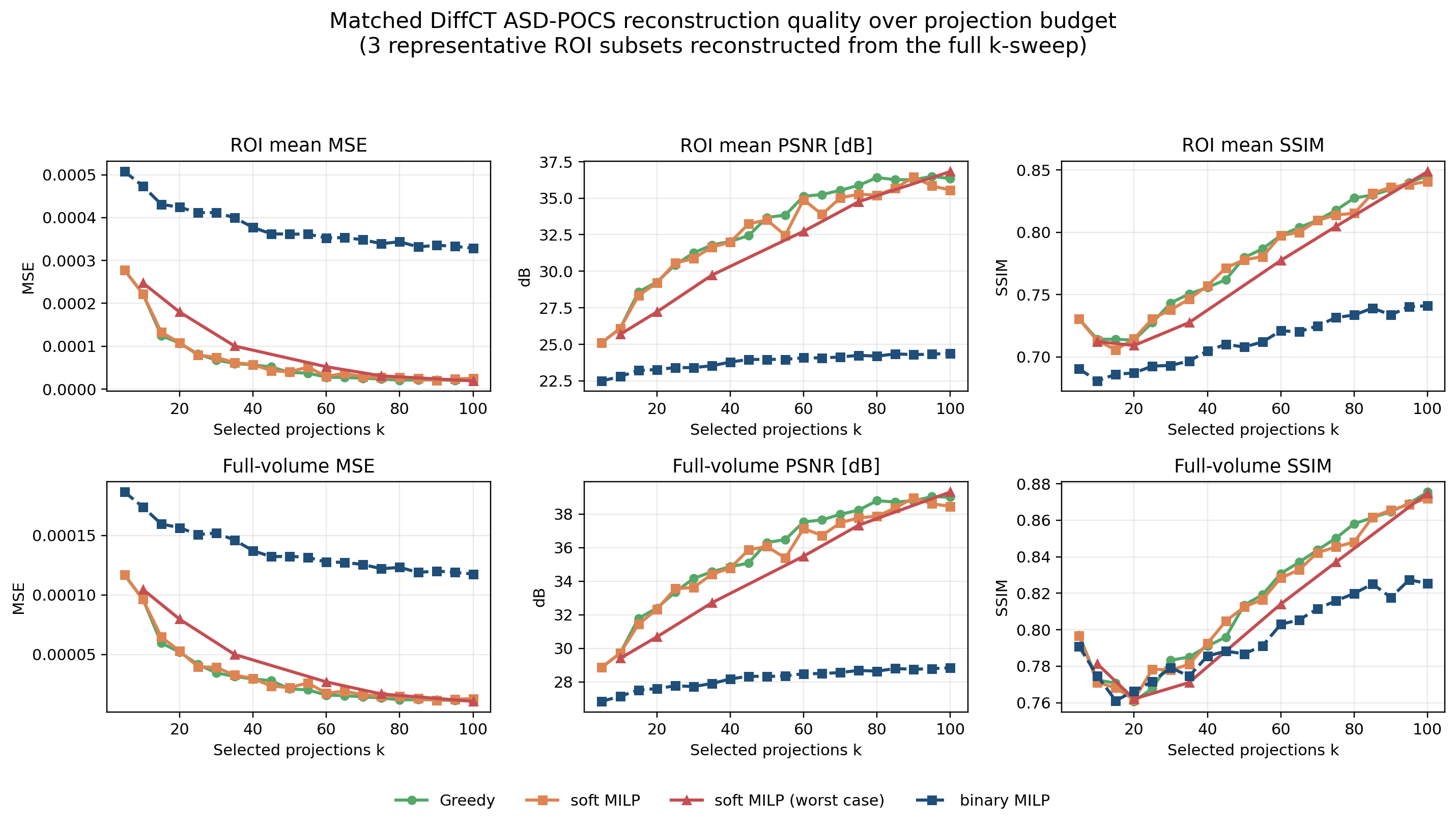}
  \caption{Matched reconstruction quality in the main single-ROI comparison, shown for MSE, PSNR, and three-dimensional SSIM. All subset reconstructions use the same geometry, grid, and reconstruction parameters as the 800-view reference volume.}
  \label{fig:k_reconstruction_20260410}
\end{figure}

The budget progression in \cref{fig:roi03_roi04_soft_milp_budget_progression} illustrates how this transfer extends across occlusion stages. Reading along each row, most of the visual gain concentrates in the lower-budget columns and the incremental improvement shrinks towards $k=100$, again reflecting the submodular structure. The more informative comparison is across rows: the number of projections required for visually coherent ROI recovery increases with occlusion severity. At the unoccluded and mild stages, $k=60$ already yields near-complete regional recovery, while under moderate occlusion a comparable result is first achieved at $k=100$. This shift occurs because validity constraints exclude well-positioned but diagnostically corrupted projections, and a larger raw budget is therefore required to assemble an equally complete set of uncorrupted views. Importantly, at the moderate stage the pipeline still achieves coherent and visually clean ROI recovery at $k=100$ despite a substantial number of external occluders, confirming that the validity-aware selection concentrates the projection budget effectively even under significant attenuation.

\begin{figure}[htbp]
  \centering
  \includegraphics[width=\linewidth]{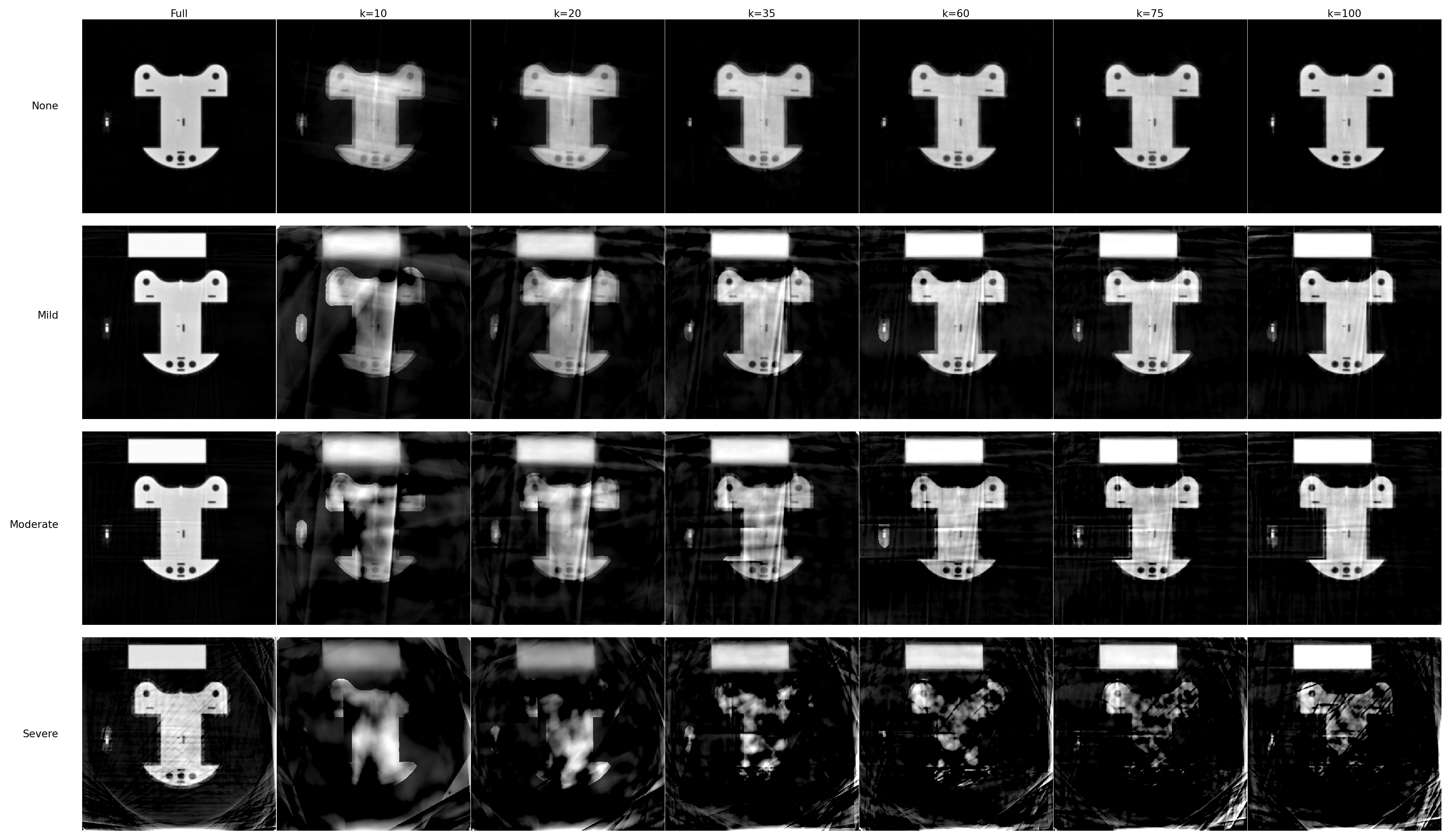}
  \caption{Budget progression of the centre axial section using the soft MILP trajectory for a representative target region across all four occlusion stages (rows) and projection budgets from 10 to 100 (columns), alongside the 800-view reference. Heavier occlusion requires a larger projection budget to achieve comparable ROI recovery: at the unoccluded and mild stages $k=60$ is near-sufficient, while under moderate occlusion coherent recovery requires $k=100$. The moderate stage at $k=100$ still produces visually clean regional reconstruction despite the additional occluders, illustrating that the validity-constrained selection concentrates projections on uncorrupted viewing angles. Intensities were robustly calibrated to the positive $99.5^\mathrm{th}$ percentile plateau of the non-occluded ASD-POCS reference.}
  \label{fig:roi03_roi04_soft_milp_budget_progression}
\end{figure}

\subsection{ESR-to-Reconstruction Correlation}

\cref{tab:esr_reco_corr_20260410} reports Pearson $r$ and Spearman $\rho$ between ESR and ROI reconstruction quality within each occlusion stage, for both mean ESR and the $Q_{0.95}$ directional ESR. The two variants tell a consistent but complementary story.

Mean ESR is a strong predictor of ROI MSE and PSNR under mild and moderate occlusion, where the method separation is large enough that the directional coverage differences translate cleanly into image-domain quality differences. At the unoccluded stage the association is weaker, and this is not a failure of the metric but a consequence of how ESR is constructed. ESR is evaluated over $\mathcal{S}_{\text{valid}}(v)$, the set of sources that pass the attenuation validity filter. Without occlusion, nearly all candidate sources remain valid, the valid sets of every method are nearly identical, and the soft methods cluster tightly in ESR space. There is simply less signal to predict, since all reasonable trajectories cover the ROI adequately and the resulting image-domain differences are small. Under occlusion, the validity filter becomes load-bearing, progressively excluding exactly the directions whose projections would introduce streaking artefacts if reconstructed through heavily attenuating paths. ESR then captures not only geometric directional coverage but the coverage that is actually diagnostically usable, which is why its predictive strength grows in lockstep with occlusion severity. The weaker correlation at the unoccluded stage is therefore a sign that the validity-aware formulation is working correctly. Under severe occlusion the Pearson correlation for MSE weakens again, while the rank-based Spearman coefficient remains high, pointing to a non-linear relationship where binary MILP becomes such a strong outlier that linear correlation understates the true monotone association. The $Q_{0.95}$ directional ESR shows modestly stronger Pearson correlations at the severe stage, particularly for PSNR and SSIM. This is interpretable, since at high attenuation isolated directional holes (captured by the 95th-percentile gap rather than the mean) are the primary driver of localized streaking artefacts, so the tail of the angular-gap distribution carries extra predictive information that the mean does not.


SSIM behaves differently from MSE and PSNR throughout. Its correlation with ESR is weaker and less consistent across stages, which is expected since SSIM is sensitive to structural patterns and contrast modulation, both of which are also shaped by the ASD-POCS regularization and not by directional coverage alone. ESR therefore predicts SSIM less reliably than the pixel-domain metrics.

\cref{fig:sphere_coverage_20260410} provides a geometric view of what drives these correlations. The Mollweide projection of per-direction angular gaps shows that at low budget, large connected regions of direction space remain severely undersampled, and the mean angular gap is far above the Nyquist threshold. At high budget the gaps shrink and become sparse and the sphere is nearly uniformly covered. The direction-space gaps propagate into mean ESR via the relation $\widehat{f}_{\mathrm{res}}(v) = 2r\,\bar{\varphi}(v)$, linking the geometric sampling structure directly to the image-domain quality differences the correlation table captures.

\begin{table}[htbp]
  \centering
  \caption{Stage-wise Pearson \(r\) / Spearman \(\rho\) between effective spatial resolution and ROI reconstruction metrics. Left columns use mean ESR; right columns use the $Q_{0.95}$ directional ESR.}
  \label{tab:esr_reco_corr_20260410}
  \resizebox{\linewidth}{!}{\input{content/assets/experiments_20260410/tables/table_esr_reco_correlations_stagewise_compact.tex}}
\end{table}

\begin{figure}[htbp]
  \centering
  \includegraphics[width=\linewidth]{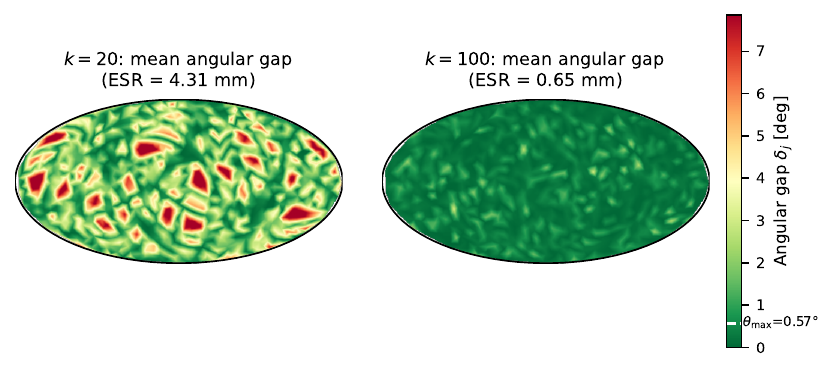}
  \caption{Per-direction angular sampling gap $\delta_j$ on the Mollweide projection of the unit sphere for a representative target region, for $k=20$ (left) and $k=100$ (right) soft Greedy selections. Colour encodes the angular gap in degrees; red indicates poor directional coverage. The dashed white line on the colour bar marks the Nyquist-motivated threshold $\theta_\mathrm{max}=f_\mathrm{min}/(2r)=0.57^{\circ}$ for the studied geometry. Panel titles state the mean angular gap and the resulting ESR.}
  \label{fig:sphere_coverage_20260410}
\end{figure}

The completeness measures are coherent and predictive at the level of broad trends, but matched reconstruction remains necessary whenever competing methods are separated by only small differences.

\subsection{Multi-Target Results}

The multi-target results (\cref{tab:multi_roi_reco_20260410}) show that the framework extends coherently to the joint setting, with the spatial arrangement of targets as the dominant factor governing the cost of serving multiple regions simultaneously.

In the clustered comparison, single-ROI soft Greedy and soft MILP set the per-region quality ceiling, since each trajectory is dedicated to a single target. Multi-ROI soft MILP operates under a strictly harder constraint and pays a quality penalty that scales with cluster spread. The six ROI centres span pairwise distances of $16$--$50\,\mathrm{mm}$ (mean $30\,\mathrm{mm}$), and this geometry drives the result. Cluster~C, whose ROIs lie within a maximum pairwise separation of $35\,\mathrm{mm}$, incurs a loss of less than $1\,\mathrm{dB}$ relative to the single-ROI best. Cluster~B, which spans up to $50\,\mathrm{mm}$ and includes ROIs at geometrically opposing positions, loses more than $2\,\mathrm{dB}$. The penalty is therefore a geometric effect of target spread, not a general limitation of the joint formulation. Multi-ROI binary MILP is insensitive to cluster geometry throughout, producing nearly identical scores across all three clusters, confirming that its coverage model distributes projections uniformly rather than adapting to the spatial arrangement of targets.

In the scaling study, quality for the soft methods is broadly stable as the number of simultaneously served regions grows from one to six. Multi-ROI soft MILP marginally improves over single-ROI optimisation at intermediate counts, consistent with joint optimisation exploiting shared directional support when targets are well-distributed. A quality dip at $q=5$ affects all soft methods equally and persists into $q=6$ for multi-ROI soft MILP, pointing to a geometrically conflicting target subset rather than a monotone capacity limit of the framework. Multi-ROI binary MILP is insensitive to $q$ but remains consistently $5$--$6\,\mathrm{dB}$ below the soft methods across the entire range.

The fusion study addresses the practically relevant case in which a single shared trajectory must serve all regions. The ``off (no fusion)'' baseline means independent per-ROI trajectories for soft Greedy (not realisable as a single acquisition) and unfused joint optimisation for multi-ROI soft MILP; the meaningful comparison is therefore between the unfused joint baseline and the fused variants. Distance-weighted fusion improves over the unfused baseline for both methods and, crucially, a single fused trajectory matches or exceeds the quality of independent per-ROI acquisitions. Uniform weighting, which ignores the spatial spread of up to $50\,\mathrm{mm}$ between targets, costs approximately $2\,\mathrm{dB}$ relative to the unfused baseline. The choice of aggregation strategy therefore matters more than the choice between greedy and exact optimisation in this setting, and distance-weighted fusion is the recommended approach whenever one acquisition must cover multiple spatially distributed regions.

\begin{table}[htbp]
  \centering
  \caption{Multi-target reconstruction quality at $k=25$ projections (ROI mean PSNR [dB]).
  Top: clustered comparison.
  Centre: scaling with the number of simultaneously served regions.
  Bottom: fusion strategy comparison.}
  \label{tab:multi_roi_reco_20260410}
  \resizebox{\linewidth}{!}{%
  \begin{tabular}{llcccc}
    \hline
    \multicolumn{2}{l}{\textit{Clustered comparison}} & Cluster A & Cluster B & Cluster C & Mean \\
    \hline
    & soft Greedy             & 26.89 & 28.98 & 29.29 & 28.39 \\
    & soft MILP               & 26.70 & 29.27 & 29.44 & 28.47 \\
    & Multi-ROI soft MILP     & 26.82 & 26.93 & 28.67 & 27.47 \\
    & Multi-ROI binary MILP   & 22.63 & 22.64 & 22.60 & 22.62 \\
    \hline
    \multicolumn{2}{l}{\textit{Scaling ($q$ = no.\ of ROIs served)}} & $q=1$ & $q=3$ & $q=5$ & $q=6$ \\
    \hline
    & soft Greedy             & 28.27 & 28.22 & 26.89 & 28.62 \\
    & soft MILP               & 28.44 & 28.99 & 26.70 & 28.34 \\
    & Multi-ROI soft MILP     & 28.43 & 29.21 & 26.82 & 26.80 \\
    & Multi-ROI binary MILP   & 23.36 & 22.69 & 22.63 & 22.63 \\
    \hline
    \multicolumn{2}{l}{\textit{Fusion strategy ($q=6$)}} & \multicolumn{2}{c}{soft Greedy} & \multicolumn{2}{c}{Multi-ROI soft MILP} \\
    \hline
    & Off (no fusion)         & \multicolumn{2}{c}{28.62} & \multicolumn{2}{c}{28.40} \\
    & Distance-weighted       & \multicolumn{2}{c}{29.56} & \multicolumn{2}{c}{29.80} \\
    & All ROIs, uniform       & \multicolumn{2}{c}{26.56} & \multicolumn{2}{c}{26.80} \\
    \hline
  \end{tabular}}
\end{table}

\FloatBarrier

%% file: content/assets/experiments_20260410/tables/table_k_checkpoints_reconstruction.tex
\begin{tabular}{llccccccc}
\hline
$k$ & Method & Binary Tuy & Soft-Tuy & ESR$_{mean}$ [mm] & ROI MSE & ROI PSNR [dB] & ROI SSIM & Opt.\ gap (\%) \\
\hline
20 & soft Greedy & 0.268 $\pm$ 0.003 & 0.164 $\pm$ 0.003 & 4.26 $\pm$ 0.20 & 0.00011 & 29.25 & 0.713 & -- \\
20 & soft MILP & 0.269 $\pm$ 0.003 & 0.164 $\pm$ 0.003 & 4.26 $\pm$ 0.20 & 0.00011 & 29.21 & 0.714 & 0.99 \\
20 & soft MILP (worst case) & 0.268 $\pm$ 0.003 & 0.164 $\pm$ 0.003 & 4.26 $\pm$ 0.20 & 0.00018 & 27.21 & 0.709 & n/a$^*$ \\
20 & binary MILP & 0.248 $\pm$ 0.014 & 0.124 $\pm$ 0.004 & 6.66 $\pm$ 1.30 & 0.00042 & 23.27 & 0.687 & 0.56 \\
\hline
60 & soft Greedy & 0.607 $\pm$ 0.009 & 0.384 $\pm$ 0.004 & 1.28 $\pm$ 0.03 & 0.00003 & 35.13 & 0.797 & -- \\
60 & soft MILP & 0.609 $\pm$ 0.008 & 0.385 $\pm$ 0.003 & 1.28 $\pm$ 0.03 & 0.00003 & 34.90 & 0.797 & 1.00 \\
60 & soft MILP (worst case) & 0.607 $\pm$ 0.009 & 0.384 $\pm$ 0.004 & 1.28 $\pm$ 0.03 & 0.00005 & 32.73 & 0.777 & n/a$^*$ \\
60 & binary MILP & 0.518 $\pm$ 0.049 & 0.279 $\pm$ 0.023 & 3.39 $\pm$ 1.09 & 0.00035 & 24.08 & 0.721 & 0.86 \\
\hline
100 & soft Greedy & 0.793 $\pm$ 0.005 & 0.530 $\pm$ 0.004 & 0.67 $\pm$ 0.02 & 0.00002 & 36.34 & 0.845 & -- \\
100 & soft MILP & 0.797 $\pm$ 0.004 & 0.533 $\pm$ 0.003 & 0.66 $\pm$ 0.01 & 0.00003 & 35.55 & 0.841 & 2.05 \\
100 & soft MILP (worst case) & 0.793 $\pm$ 0.005 & 0.530 $\pm$ 0.004 & 0.67 $\pm$ 0.02 & 0.00002 & 36.82 & 0.849 & n/a$^*$ \\
100 & binary MILP & 0.629 $\pm$ 0.087 & 0.364 $\pm$ 0.033 & 2.72 $\pm$ 1.38 & 0.00033 & 24.38 & 0.741 & 0.77 \\
\hline
\multicolumn{9}{l}{\footnotesize $^*$Gurobi bound not applicable to the max-min objective form.} \\
\end{tabular}

%% file: content/assets/experiments_20260410/tables/table_cross_stage_selection_k100.tex
\begin{tabular}{llcccc}
\hline
Stage & Method & Binary Tuy & Soft-Tuy & ESR$_{mean}$ [mm] & Opt.\ gap (\%) \\
\hline
None & soft Greedy & 0.793 $\pm$ 0.005 & 0.530 $\pm$ 0.004 & 0.67 $\pm$ 0.02 & -- \\
 & soft MILP & 0.797 $\pm$ 0.004 & 0.533 $\pm$ 0.003 & 0.66 $\pm$ 0.01 & 2.05 \\
 & binary MILP & 0.629 $\pm$ 0.087 & 0.364 $\pm$ 0.033 & 2.72 $\pm$ 1.38 & 0.77 \\
\hline
Mild & soft Greedy & 0.785 $\pm$ 0.003 & 0.522 $\pm$ 0.004 & 0.71 $\pm$ 0.01 & -- \\
 & soft MILP & 0.790 $\pm$ 0.009 & 0.527 $\pm$ 0.003 & 0.68 $\pm$ 0.02 & 1.49 \\
 & binary MILP & 0.619 $\pm$ 0.082 & 0.362 $\pm$ 0.035 & 2.94 $\pm$ 1.46 & 0.68 \\
\hline
Moderate & soft Greedy & 0.781 $\pm$ 0.002 & 0.519 $\pm$ 0.004 & 0.71 $\pm$ 0.03 & -- \\
 & soft MILP & 0.784 $\pm$ 0.005 & 0.520 $\pm$ 0.003 & 0.71 $\pm$ 0.03 & 1.31 \\
 & binary MILP & 0.475 $\pm$ 0.157 & 0.286 $\pm$ 0.084 & 11.40 $\pm$ 9.36 & 0.16 \\
\hline
Severe & soft Greedy & 0.753 $\pm$ 0.004 & 0.493 $\pm$ 0.006 & 0.81 $\pm$ 0.03 & -- \\
 & soft MILP & 0.759 $\pm$ 0.002 & 0.495 $\pm$ 0.006 & 0.81 $\pm$ 0.02 & 1.02 \\
 & binary MILP & 0.426 $\pm$ 0.147 & 0.258 $\pm$ 0.082 & 14.54 $\pm$ 8.43 & 0.13 \\
\hline
\end{tabular}

%% file: content/assets/experiments_20260410/tables/table_greedy_milp_certificates_occlusion.tex
\begin{tabular}{lrrrrrrr}
\hline
Occlusion & Pairs & Greedy=opt. & MILP opt. & Median $f_G/f_I$ & Min. $f_G/f_I$ & Median $f_G/U$ & Min. $f_G/U$ \\
\hline
none & 60 & 17 (28.3\%) & 49 (81.7\%) & 0.995 & 0.992 & 0.986 & 0.972 \\
mild & 60 & 15 (25.0\%) & 53 (88.3\%) & 0.997 & 0.989 & 0.988 & 0.972 \\
moderate & 60 & 22 (36.7\%) & 55 (91.7\%) & 0.999 & 0.991 & 0.989 & 0.977 \\
severe & 60 & 18 (30.0\%) & 58 (96.7\%) & 0.999 & 0.991 & 0.989 & 0.982 \\
pooled & 240 & 72 (30.0\%) & 215 (89.6\%) & 0.998 & 0.989 & 0.989 & 0.972 \\
\hline
\end{tabular}

%% file: content/assets/experiments_20260410/tables/table_greedy_binary_vs_soft.tex
\begin{tabular}{llcccccc}
\hline
$k$ & Method & Binary Tuy & Soft-Tuy & Sat. Cov. & ESR$_{mean}$ [mm] & ROI PSNR [dB] & ROI SSIM \\
\hline
20 & binary soft Greedy & 0.248 $\pm$ 0.014 & 0.124 $\pm$ 0.004 & 0.248 $\pm$ 0.014 & 6.66 $\pm$ 1.30 & 23.28 & 0.688 \\
20 & soft Greedy & 0.268 $\pm$ 0.003 & 0.164 $\pm$ 0.003 & 0.167 $\pm$ 0.003 & 4.26 $\pm$ 0.20 & 29.25 & 0.713 \\
\hline
60 & binary soft Greedy & 0.514 $\pm$ 0.048 & 0.278 $\pm$ 0.022 & 0.514 $\pm$ 0.048 & 3.41 $\pm$ 1.11 & 24.02 & 0.719 \\
60 & soft Greedy & 0.607 $\pm$ 0.009 & 0.384 $\pm$ 0.004 & 0.415 $\pm$ 0.003 & 1.28 $\pm$ 0.03 & 35.13 & 0.797 \\
\hline
100 & binary soft Greedy & 0.626 $\pm$ 0.083 & 0.367 $\pm$ 0.036 & 0.626 $\pm$ 0.083 & 2.74 $\pm$ 1.36 & 24.32 & 0.741 \\
100 & soft Greedy & 0.793 $\pm$ 0.005 & 0.530 $\pm$ 0.004 & 0.601 $\pm$ 0.002 & 0.67 $\pm$ 0.02 & 36.34 & 0.845 \\
\hline
\end{tabular}

%% file: content/assets/experiments_20260410/tables/table_esr_reco_correlations_stagewise_compact.tex
\begin{tabular}{lrrrrrrrr}
\hline
 & & \multicolumn{3}{c}{Mean ESR} & \multicolumn{3}{c}{$Q_{0.95}$ directional ESR} \\
Occlusion & $n$ & ROI MSE $r/\rho$ & ROI PSNR $r/\rho$ & ROI SSIM $r/\rho$ & ROI MSE $r/\rho$ & ROI PSNR $r/\rho$ & ROI SSIM $r/\rho$ \\
\hline
None & 150 & 0.652 / 0.685 & -0.579 / -0.685 & -0.609 / -0.856 & 0.704 / 0.758 & -0.598 / -0.759 & -0.611 / -0.886 \\
Mild & 140 & 0.979 / 0.931 & -0.617 / -0.838 & -0.398 / -0.709 & 0.970 / 0.946 & -0.637 / -0.872 & -0.402 / -0.724 \\
Moderate & 140 & 0.933 / 0.783 & -0.611 / -0.943 & -0.209 / -0.749 & 0.879 / 0.802 & -0.705 / -0.917 & -0.282 / -0.607 \\
Severe & 140 & 0.577 / 0.788 & -0.584 / -0.859 & -0.618 / -0.877 & 0.683 / 0.757 & -0.712 / -0.814 & -0.757 / -0.845 \\
\hline
\end{tabular}

%% file: content/ablation_study.tex

\section{Ablation Study}\label{sec:ablation}

\cref{tab:ablation_20260410} summarizes the controlled ablation studies using mean effective spatial resolution as a common diagnostic. Each experiment varies one modelling component at a time while holding the remainder of the selection framework fixed. The deterministic target pool introduced in \cref{sec:experiments} ensures that observed changes reflect the ablated parameter rather than incidental differences in ROI sampling.

\begin{table}[htbp]
  \centering
  \caption{Controlled ablation results: mean effective spatial resolution [mm] at $k=20$ projections on the fixed six-ROI pool. Each group varies one modelling component independently; all other settings are held at their defaults.}
  \label{tab:ablation_20260410}
  \resizebox{\linewidth}{!}{%
  \begin{tabular}{llccc}
    \hline
    Ablation & Condition & soft Greedy & soft MILP & Binary MILP \\
    \hline
    \multirow{3}{*}{Target resolution}
      & $f_{\min} = 0.5$\,mm & 4.51 & 4.52 & 7.44 \\
      & $f_{\min} = 1.0$\,mm & 4.51 & 4.44 & 7.44 \\
      & $f_{\min} = 2.0$\,mm & 4.42 & 4.36 & 7.44 \\
    \hline
    \multirow{3}{*}{Attenuation validity}
      & lenient  & 4.39 & 4.36 & 4.88 \\
      & default  & 4.39 & 4.36 & 7.44 \\
      & strict   & 4.39 & 4.36 & 28.88 \\
    \hline
    \multirow{3}{*}{Sphere discretisation}
      & 800 directions  & 4.44 & 4.40 & 7.44 \\
      & 1200 directions & 4.33 & 4.32 & 7.44 \\
      & 2000 directions & 4.58 & 4.59 & 7.44 \\
    \hline
  \end{tabular}}
\end{table}

\paragraph{Target resolution.}
When the prescribed feature size is reduced, the completeness requirement becomes stricter and the soft coverage objective reflects a tighter angular tolerance, so the achievable ESR tends to worsen. The degree to which this propagates differs across methods. Soft MILP shows the clearest monotone response, with mean ESR decreasing from $4.52$\,mm at $f_\mathrm{min}=0.5$\,mm to $4.36$\,mm at $f_\mathrm{min}=2.0$\,mm, consistent with the relaxed tolerance admitting a wider range of contributing views. Soft Greedy shows a flatter response ($4.51$, $4.51$, $4.42$\,mm), reflecting that the marginal-gain selection is dominated by geometric spread and is relatively insensitive to the exact angular tolerance at this budget. Binary MILP selects identical projections across all three $f_\mathrm{min}$ values ($7.44$\,mm throughout), as its binary hit-or-miss model reduces to maximising the count of covered plane normals and the maximally spread 20-view set is determined primarily by geometric coverage, which is insensitive to the exact angular threshold at this cardinality.

\paragraph{Attenuation validity.}
The validity study reveals a threshold effect that is specific to the binary formulation. Soft Greedy and soft MILP produce identical ESR values under all three validity settings ($4.39$ and $4.36\,\mathrm{mm}$ respectively in every condition), showing that the continuous soft objective is insensitive to the tested validity thresholds as long as a sufficient number of valid candidates remains. For binary MILP the effect is entirely different. Under lenient validity, binary MILP achieves $4.88\,\mathrm{mm}$, within $0.5\,\mathrm{mm}$ of the soft methods, showing that the binary model's intrinsic gap relative to the soft formulation is modest when the candidate pool is unrestricted. Tightening to the default threshold widens binary MILP to $7.44\,\mathrm{mm}$ and strict thresholding collapses it further to $28.88\,\mathrm{mm}$, a more-than-fivefold degradation from lenient to strict. Validity constraints are therefore the primary amplifier of binary MILP's underperformance, since the hard hit-or-miss threshold means a plane normal registers as covered only if at least one remaining valid view achieves a full angular hit, so excluding even a moderate fraction of candidate projections leaves entire families of plane normals uncovered.

\paragraph{Directional sampling density.}
The choice of sphere discretisation influences the absolute values of effective spatial resolution only moderately. For soft Greedy and soft MILP, the best mean ESR is obtained with 1200 sampled directions ($4.33$ and $4.32$\,mm). Both 800 and 2000 directions give slightly worse results, with the 2000-direction case ($4.58$ and $4.59$\,mm) performing worse than 800 directions ($4.44$ and $4.40$\,mm). This non-monotone behaviour is consistent with over-discretisation increasing the number of nearly collinear direction pairs and thereby diluting the soft score signal that the greedy marginal-gain criterion relies on. The qualitative conclusions are therefore not an artefact of any particular sphere discretisation, and the 1200-direction default is retained throughout.

\paragraph{Worst-case soft objective.}
Replacing the mean saturated coverage objective with the max-min formulation shifts the operating point of the optimizer. Worst-case soft MILP raises the minimum directional coverage floor at the cost of a slightly lower mean. On the ESR metric, the worst-case formulation produces values indistinguishable from soft Greedy at matched budgets because both methods achieve similar average directional support (see \cref{tab:k_checkpoints_reco_20260410}); the difference becomes visible in reconstruction quality under occlusion, where the worst-case formulation shows the strongest artefact suppression. \cref{fig:qualitative_reconstruction_20260410} illustrates this at the moderate occlusion stage. At $k=60$ the worst-case formulation already shows marginally fewer structural artefacts than the mean-objective variants, and by $k=100$ all three soft formulations converge to a visually similar result.

\begin{figure}[htbp]
  \centering
  \includegraphics[width=\linewidth]{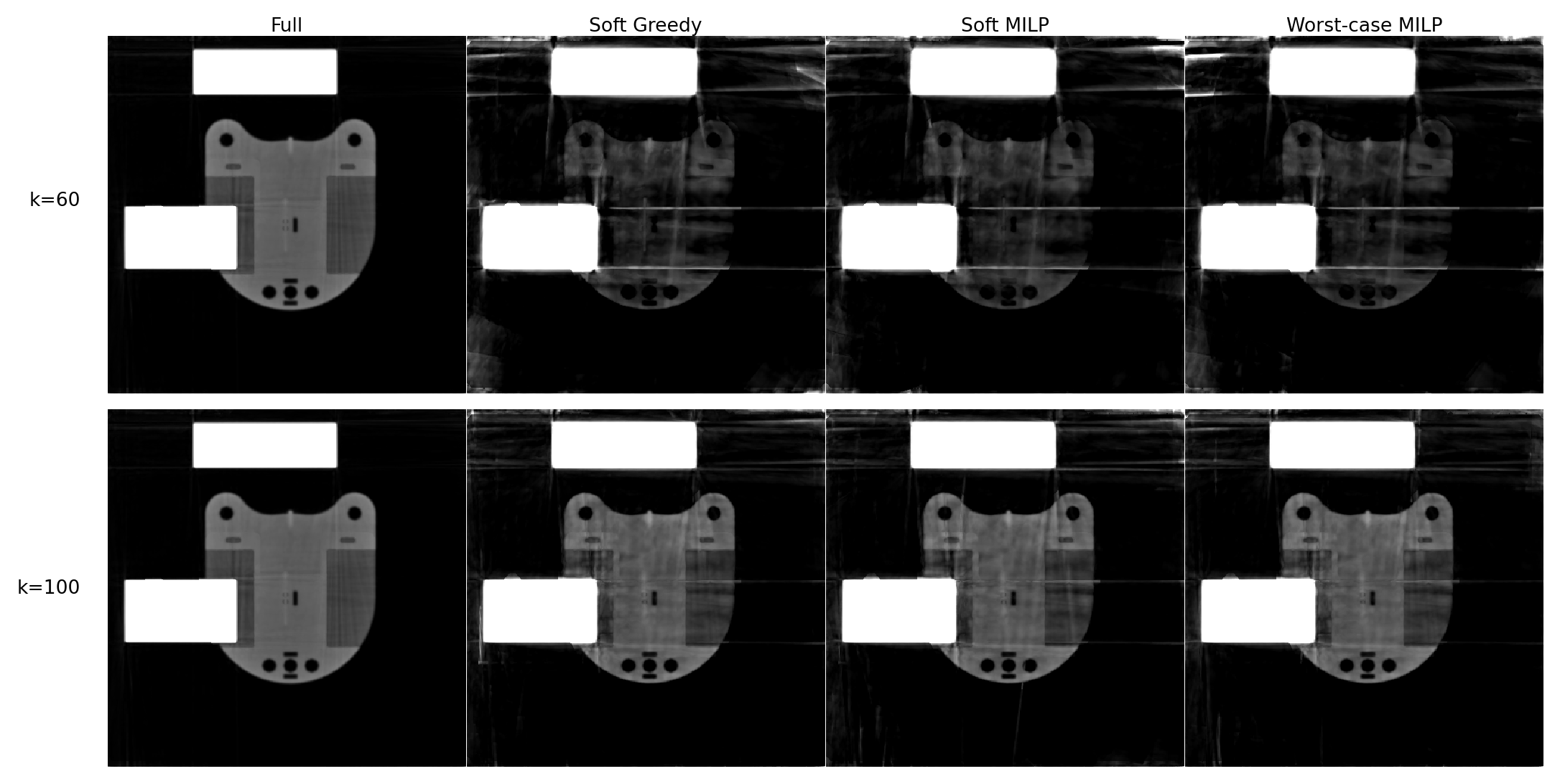}
  \caption{Qualitative comparison of the three soft formulations for a representative target region, moderate occlusion stage, at $k=60$ (top row) and $k=100$ (bottom row). The worst-case soft MILP formulation shows marginally fewer structural artefacts than the mean-objective variants at $k=60$; by $k=100$ all three converge to a visually similar result close to the full-view reference.}
  \label{fig:qualitative_reconstruction_20260410}
\end{figure}

Taken together, these ablations demonstrate that the model responds coherently to changes in resolution demand, attenuation-based admissibility, directional sampling, and objective formulation. Absolute values shift as expected, while the relative ordering between the soft methods remains stable across all conditions.

\FloatBarrier

%% file: content/discussion.tex

\section{Discussion}\label{sec:discussion}

Whether a projection is selected is a binary decision; whether a plane-normal direction is covered should not be. That design choice between continuous and binary completeness turns out to matter more than the choice between greedy and exact optimization, and its consequences are amplified nearly sixfold by the attenuation-based validity constraints that distinguish realistic acquisition from idealized settings. Across the studied range of budgets, occlusion levels, and target configurations, the soft coverage objective reflects genuine acquisition difficulty, a fast greedy procedure is already near-optimal, and effective spatial resolution derived from the same model predicts reconstruction quality consistently with the image-domain evidence.

\paragraph{The completeness framework captures genuine acquisition difficulty.}
The selection framework behaves coherently across all experimental perturbations. Varying the projection budget, the target feature size, the attenuation validity threshold, or the occlusion level produces responses in the expected direction throughout, arguing that the framework captures a real geometric property of the acquisition problem rather than a numerical artefact of the optimization model. In particular, stricter resolution demands and stronger attenuation-based validity constraints both reduce the effective directional support available to the optimizer, as expected from the sampling model.

\paragraph{Effective spatial resolution bridges geometry and image quality.}
Although effective spatial resolution is not the quantity optimized by the main formulations, it shows a coherent relation to matched reconstruction quality. The relation is clearest for MSE and PSNR within each occlusion stage, where poorer effective spatial resolution corresponds to higher error and lower PSNR. SSIM is more conditional. In the non-occluded stage, the attenuation-aware component of ESR has little opportunity to discriminate between views, and SSIM is dominated by sparse-view reconstruction effects and projection count. Under mild, moderate, and severe occlusion, however, ESR becomes strongly predictive of SSIM as well. This is an important transparency point. ESR should be interpreted as an attenuation-aware trajectory quality measure, not as a universal image-quality proxy that must explain every metric in every acquisition regime.

Mean ESR is used as the primary diagnostic throughout this work, even though the $Q_{0.95}$ directional ESR shows modestly stronger Pearson correlations under severe occlusion. The reason is that mean ESR is the natural counterpart to the optimization objective. The soft formulation maximizes saturated directional coverage, which is an average-case quantity over all sampled directions, and mean ESR summarizes the same average angular support in physically interpretable units. $Q_{0.95}$ ESR is a tail statistic that is more sensitive to isolated directional gaps and therefore more variable across instances, making it a less stable summary for comparing trajectories that differ primarily in average coverage. The stronger $Q_{0.95}$ correlation under severe occlusion reflects a specific regime in which validity-driven projection exclusion creates sparse, clustered gaps rather than uniform undersampling; in that regime $Q_{0.95}$ ESR provides additional diagnostic signal about tail risk in the trajectory, complementing the mean rather than replacing it.

\paragraph{Binary and continuous completeness are complementary, not interchangeable.}
The soft near-orthogonality score and saturated coverage objective form the primary optimization model in this work. They retain graded directional support, enable near-optimal greedy selection as the MILP certificates confirm, and consistently yield the strongest coverage scores and reconstruction quality throughout. The binary formulation counts a direction as either covered or not, which is a stronger hard-completeness guarantee but discards graded information once the threshold is met. The gap is visible when binary MILP's selected projections are evaluated on the soft coverage scale. At 100 projections, binary MILP reaches Soft Tuy $0.364 \pm 0.033$ against $0.533 \pm 0.003$ for soft MILP, a gap of $0.169$ that does not appear in the binary Tuy scores alone. In the image domain the soft methods dominate throughout: at $k=100$, soft methods achieve $35$--$37\,\mathrm{dB}$ in ROI PSNR while binary methods reach only $24\,\mathrm{dB}$, confirming that the continuous coverage advantage propagates directly to reconstruction quality. The paired comparison between binary Greedy and soft Greedy confirms that this gap is rooted in the objective formulation itself and is not an artefact of exact optimization. Attenuation-based validity constraints further amplify the binary model's intrinsic disadvantage nearly sixfold relative to an unconstrained setting, as the controlled ablation in \cref{sec:ablation} shows.

\paragraph{Greedy selection is stronger than its heuristic status implies.}
The MILP certificates make the quality of the greedy approximation precise. Across 240 paired budget comparisons in the main study, the pooled median ratio between the soft Greedy and soft MILP objective was 0.998, with 72 of 240 cases certified globally optimal. This agreement reflects a structural property of the soft saturated coverage objective. The marginal-gain selection rule, which recomputes the additional coverage contributed by every remaining view at each step, captures nearly all available gain because the objective has enough smoothness and complementarity for constructive selection to be near-optimal. The MILP therefore serves primarily as a certified quality benchmark for the soft formulation, defining the reference optimum and confirming that the gap between greedy and optimal is negligible on the studied CT instances. For the binary formulation the situation differs, as branch-and-cut can still materially improve hard completeness and effective spatial resolution over binary Greedy, so the MILP remains an active optimizer in that setting.

\paragraph{Full-view acquisition is not always the best target under occlusion.}
The occlusion-stress experiment clarifies why trajectory optimization remains meaningful even when it deliberately discards projections. In an unoccluded setting, an 800-view reconstruction is a natural consistency reference. Under moderate or severe occlusion, however, more views also means more corrupted measurements, since the full-view reconstruction integrates occlusion-induced streaks and flow artefacts that are absent from a validity-constrained selected trajectory. The validity filter is therefore not a limitation of the framework but an active design choice. By excluding projections whose attenuation exceeds the admissibility threshold, the selected trajectory avoids accumulating the class of artefacts that corrupt the full-view reference under heavy occlusion. The practical implication is that a compact set of valid projections should be understood not as a sparse-view approximation to full acquisition but as a mechanism for concentrating the projection budget on measurements that carry diagnostic signal rather than occlusion noise.

\paragraph{Limitations.}
The study also has clear limits. All results are derived from a single simulated object, so the evidence supports internal consistency and controlled method comparison more strongly than broad generalization. The reconstruction metrics are measured relative to a high-quality iterative reference rather than to an analytical ground truth, which is a practical and defensible choice but not a substitute for object-level accuracy. Furthermore, reconstruction follow-up was restricted to the study families most directly relevant to image quality, so some ablation trends are validated only through the selection metrics. Finally, the worst-case and binary MILP formulations are harder to certify at larger budgets, and the extent to which tighter optimality requirements would alter the conclusions remains an open question.

Within these bounds, the combined evidence supports completeness-driven projection selection as a principled approach for ROI-focused cone-beam CT. The soft coverage objective is near-optimally approximable by a fast greedy rule, certifiable by exact MILP bounds, and sufficiently expressive to capture the boundary between binary and continuous completeness in acquisition-relevant terms. For the soft formulation the MILP serves as a quality certificate; for the binary formulation it remains an active optimizer. Extending the framework to real scanner hardware, heterogeneous materials, and dynamic acquisition geometries is the natural next step toward clinical applicability.

%% file: content/conclusion.tex

\section{Conclusion}\label{sec:conclusion}

Tuy's completeness condition establishes whether exact cone-beam reconstruction is theoretically possible, but its binary character provides no gradient and no quantitative measure of how much information a constrained trajectory supplies for a given region of interest. This work addressed that gap by reformulating completeness as a continuous, resolution-aware, attenuation-constrained optimization problem and developing the theoretical and algorithmic infrastructure needed to solve it with formal guarantees.

The four contributions set out in the introduction are substantiated by the experimental record. The soft near-orthogonality score and saturated coverage objective replace the flat binary optimization landscape with a smooth, submodular objective amenable to marginal-gain selection. The greedy algorithm achieves a pooled median ratio of $0.998$ relative to the exact MILP optimum across 240 paired comparisons, confirming that the worst-case $(1-1/e)$ approximation guarantee is far from tight on realistic CT instances. Effective spatial resolution translates angular sampling gaps into resolvable feature sizes, providing a task-agnostic, physically interpretable link between the completeness model and matched reconstruction quality across all studied budgets and occlusion levels, and independently of the reconstruction algorithm. The NP-completeness of the ROI-CTTOP decision problem, together with the corresponding decision-problem hardness for binary and soft directional coverage, explains why exact optimization is non-trivial and motivates the MILP as a certified bound that validates the greedy solution rather than competing with it. Finally, the controlled ablation establishes that the binary formulation's practical limitation is its model rather than its solver: attenuation-based validity constraints amplify the intrinsic binary gap nearly sixfold, while the soft formulation remained unchanged under the studied validity thresholds as long as enough valid candidates were available.

For the field of cone-beam CT trajectory design, these results shift the framing of the problem. The limiting factor in trajectory quality is not computational but representational: a continuous coverage objective with validity-aware angular tolerances captures the acquisition geometry more faithfully than a binary one, and the difference becomes decisive under the attenuation conditions that motivate trajectory optimization in the first place. The MILP certificate further changes how greedy selection should be interpreted: it is not a heuristic approximation that might be improved by a better solver, but a near-optimal rule whose quality is bounded from above by the exact solution.

Several directions remain open. Validation on physical scanner hardware with heterogeneous materials is the most pressing step toward clinical and industrial deployment, as simulated attenuation models cannot capture all sources of validity variation. Extension of the reconstruction evaluation to task-specific image quality criteria and direct comparison with analytical ground truths would sharpen the connection between the geometric selection model and diagnostic utility. Tighter coupling between the projection selection objective and downstream iterative reconstruction algorithms, for instance by incorporating reconstruction-informed sensitivity weights into the coverage matrix, could further improve ROI fidelity at constrained budgets. Adaptive trajectory adjustment during acquisition, where the selection model is updated online as new projections arrive, is a natural extension of the greedy rule that would be particularly valuable for interventional applications.